\documentclass[twoside]{article}

\usepackage{commands}

\begin{document}

%

%

\twocolumn[

\aistatstitle{Sketch In, Sketch Out: Accelerating both Learning and Inference for Structured Prediction with Kernels}

\aistatsauthor{ Tamim El Ahmad \And Luc Brogat-Motte \And  Pierre Laforgue \And Florence d'Alch\'e-Buc }

\aistatsaddress{  T\'el\'ecom Paris \And  CentraleSup\'elec \And University of Milan \And T\'el\'ecom Paris } ]

\begin{abstract}
    Leveraging the kernel trick in both the input and output spaces, surrogate kernel methods are a flexible and theoretically grounded solution to structured output prediction.
    If they provide state-of-the-art performance on complex data sets of moderate size (e.g., in chemoinformatics), these approaches however fail to scale.
    We propose to equip surrogate kernel methods with sketching-based approximations, applied to both the input and output feature maps.
    We prove excess risk bounds on the original structured prediction problem, showing how to attain close-to-optimal rates with a reduced sketch size that depends on the eigendecay of the input/output covariance operators.
    From a computational perspective, we show that the two approximations have distinct but complementary impacts: sketching the input kernel mostly reduces training time, while sketching the output kernel decreases the inference time.
    %
    %
    Empirically, our approach is shown to scale, achieving state-of-the-art performance on benchmark data sets where non-sketched methods are intractable.
\end{abstract}


%

\section{INTRODUCTION}
\label{sec:intro}

Ubiquitous in real-world applications, structured objects have attracted a great deal of attention in machine learning \citep{bakir2007,Gartner08,nowozin2011,deshwal2019struct}.
Depending on their role, i.e., either as input or output variables, they raise distinct challenges.
Classification and regression from structured \emph{inputs} generally rely on a continuous representation learned by a deep neural network \citep{defferrard2016gcn}, or implicitly defined through a dedicated kernel \citep{collins2001, borgwardt20}.
In contrast, structured \emph{output} prediction calls for a more involved approach, since the discrete nature of the outputs impacts the definition of the loss function \citep{nowak19a, ciliberto2020general, cabannes2021fast}, and therefore the learning problem itself.

To handle this problem, several methods have been developed to relax the combinatorial problems that appear both at training and inference.
Energy-based approaches convert structured prediction into learning a scalar score function \citep{tsochantaridis2005large, lecun2007, belanger2016structured, deshwal2019struct}.
End-to-end learning typically exploits a differentiable model, together with a differentiable loss, to run gradient descent \citep{long2015fully,niculae18,berthet2020learning}.
Surrogate methods \citep{ciliberto2020general} solve a regression problem in a Hilbert space where outputs have been implicitly embedded, shortcutting the inference during learning.
%

Rare are the methods that enjoy both scalability at learning/inference steps and statistical guarantees \citep{osokin17,cabannes2021fast}.
In this work, we focus on surrogate approaches and their implementation as kernel methods, i.e., the input output kernel regression framework \citep{cortes2005,brouard2016input}.
Recent works \cite{Ciliberto2016, ciliberto2020general} have shown that they enjoy consistency, their excess risk being governed by that of the surrogate regression.
Moreover, they are well appropriate to make prediction from one structured modality to another, since kernels can be leveraged in both the input and output spaces.
Overall, they offer a general, theoretically grounded, and simple-to-implement solution to structured prediction, providing state-of-the-art results in applications such as molecule identification \citep{Schymanski2016}.

However, contrary to deep neural networks, they do not scale neither in memory nor in time without further approximation.
The aim of this paper is to equip these methods with kernel approximations to obtain a drastic complexity reduction while maintaining their statistical properties.
Several works have highlighted the power of kernel approximations, from Random Fourier Features \citep{rahimi2007, brault2016random, rudi2017generalization, li2021unified}, to general low-rank approaches \citep{bach2013sharp, meanti2020kernel}.

In this work we focus on sketching \citep{mahoney2011randomized,woodruff14}, a general dimension reduction method based on linear random projections.  
Applied to kernel approximation, sketching 
has been widely studied through Nystr\"{o}m's sub-sampling approximation \citep{WilliamsNystromNIPS2000, elalaoui_NIPS2015, rudi2015less}, and  further explored using Gaussian or Randomized Orthogonal Systems \citep{Yang2017, lacotte2020adaptive}. Interpreted as a way to provide data-dependent random features \citep{WilliamsNystromNIPS2000, Yang_NIPS2012_621bf66d,Kpotufe2020}, this approach has allowed to scale up kernel PCA \citep{stergeNKPCA22}, kernel mean embedding \citep{chatalicaistats22,chatalicicml22} or independence tests \citep{kalinke2023nystr} while enjoying statistical guarantees.
%
However, sketching has been limited so far to scalar kernel machines. No current approach covers both sides of the coin, i.e., applying approximations to both the input and output kernels.
Motivated by surrogate structured prediction, we close this gap and make the following contributions:
\begin{itemize}[topsep=0pt,parsep=0pt,itemsep=3pt]
\item We apply sketching to the vector-valued kernel regression problem solved in structured prediction, both on inputs and outputs, which accelerates respectively learning and inference.
\item We derive excess risk bounds controlled by the properties of the sketched projection operators.
\item We prove that sub-Gaussian sketches provide \mbox{close-to-optimal} rates with small sketch sizes.
\item 
We empirically show that our algorithms maintain good accuracy on moderate-size datasets while enabling kernel surrogate methods on large datasets where the standard approach is simply intractable.
\end{itemize}

\paragraph{Notations.}
We introduce now generic notations for the input (output) space and kernel, detailed in \cref{apx:not_def}.
If $\bmZ$ denotes a generic Polish space, $\kernelz$ is a positive definite kernel over $\bmZ$ and $\psiz(z)\coloneqq\kernelz(\cdot, z)$ is the canonical feature map of $\kernelz$.
$\Hz$ denotes the Reproducing Kernel Hilbert Space (RKHS) associated to $\kernelz$.
$\SZ : f \in \Hz \mapsto (1/\sqrt{n}) (f(z_1), \ldots, f(z_n))^\top$ is the sampling operator over $\Hz$ \citep{Smale_2007}.
%

\section{BACKGROUND}
\label{sec:IOKR}

We now recall the structured prediction setting based on a kernel-induced loss, and a state-of-the-art surrogate approach to solve it.
We also provide reminders about sketching as a way to scale-up kernel methods.

\paragraph{Structured prediction with surrogate kernel methods.}
Let $\bmX$ be the input space and $\bmY$ a structured output space. In general, $\bmY$ is finite and extremely large. Define a positive definite kernel $\kernely:\bmY \times \bmY \to \reals$, that measures how close two objects from $\bmY$ are. We consider the loss function induced by $\kernely$, defined as $\ell:(y,y') \rightarrow  \|\psiy(y) - \psiy(y')\|^2_{\Hy}$.
Note that it can be computed using the kernel trick. 
%
%
%
%
%
Given an unknown joint probability distribution $\rho$ defined on $\bmX \times \bmY$, the goal of structured prediction is to approximate
\begin{equation}\label{eq:sp}
f^* =\underset{f: \bmX \rightarrow  \bmY}{\arg \min} ~ \bmR(f)\,,
\end{equation}
where $\bmR(f) = \E_{(x,y) \sim \rho}\left[\| \psiy(y) - \psiy(f(x)) \|_{\Hy}^{2} \right]$, using only an i.i.d. sample $\{(x_1,y_1), \ldots ,(x_n,y_n)\}$ drawn from $\rho$.
Estimating directly $f^*$ is not tractable, such that many works \citep{cortes2005, geurts2006, Brouard_icml11, Ciliberto2016} have proposed instead the following two-step approach:

{\bf 1. Surrogate Regression:} Find an estimator $\hat{h}$ of the surrogate target $h^* \colon x \mapsto  \E_y[\psiy(y) | x]$ such that
\[
h^* = \argmin_h ~ \mathbb{E}_{(x, y)}\left[\left\|h\left(x\right)-\psiy\left(y\right)\right\|_{\Hy}^2\right]\,.
\]

{\bf 2. Pre-image:} Define $\hat{f}$ by decoding $\hat{h}$, i.e.,
\[
\hat{f}(x) = d(\hat{h}(x)) \coloneqq \argmin _{y \in \mathcal{Y}}\big\|\hat{h}(x)-\psiy(y)\big\|_{\Hy}^2\,.
\]

The surrogate regression in Step 1 is much easier to handle than the initial structured prediction problem: it avoids learning $f$ through the composition with the implicit feature map $\psiy$, and relegates the difficulty of handling structured objects to Step 2, i.e. at inference.
In addition, vector-valued regression into infinite-dimensional spaces is a well-studied problem, that can be solved by using the kernel trick in the output space.
This two-step approach belongs to the general framework of SELF \citep{Ciliberto2016} and ILE \citep{ciliberto2020general} and enjoys valuable theoretical guarantees.
It is Fisher consistent, i.e., $h^*$ yields $f^*$ after decoding, and the excess risk of $\hat{f}$ is controlled by that of $\hat{h}$.

\paragraph{Input Output ridge Kernel Regression. }A common choice to tackle in practice the surrogate regression problem consists in solving a {\it kernel ridge regression problem}, leveraging kernels in both input and output spaces. The hypothesis space is chosen as a vector-valued Reproducing Kernel Hilbert Space (vv-RKHS) \citep{Senkene1973, micchelli2005learning, carmeli2006vector, carmeli2010vector}. 
In the same way that RKHS are based on positive symmetric definite kernels, vv-RKHS are based on Operator-Valued Kernels (OVK). In our setting,  we define an OVK $\bmK$, as a mapping $\bmK: \bmX \times \bmX \rightarrow  \bmL(\Hy)$, where $ \bmL(\Hy)$ is the set of bounded linear operators on $\Hy$, and that satisfies the properties recalled in \Cref{sec-apx:vvRKHS}.
An OVK $\bmK$ is uniquely associated with a vv-RKHS $\bmH$, i.e. a Hilbert space of functions from $\bmX$ to $\Hy$ that enjoys the reproducing kernel property (see \Cref{sec-apx:vvRKHS}).

%
%

In what follows, we opt for the identity decomposable OVK $\bmK: \bmX \times \bmX \to \bmL(\Hy)$, defined as: $\bmK\left(x, x^{\prime}\right)=\kernelx\left(x, x^{\prime}\right) I_{\Hy}$, where $\kernelx: \mathcal{X} \times \mathcal{X} \rightarrow \reals$ is a p.d. scalar-valued kernel on $\mathcal{X}$.
%
%
In {\it Input Output Kernel Ridge Regression} (IOKR for short, \citealt{Brouard_icml11, kadri2013generalized, brouard2016input, ciliberto2020general}, also introduced as Kernel Dependency Estimation by \citet{westonNIPS2002}), the estimator of the surrogate regression is obtained by solving the following Ridge regression problem within $\bmH$, given a regularisation penalty $\lambda > 0$,
\begin{equation}\label{pb:iokr-ridge}
    \hat{h} = \argmin _{h \in \bmH} ~ \frac{1}{n} \sum_{i=1}^n\left\|\psiy(y_i) - h(x_i)\right\|_{\Hy}^2+\lambda\|h\|_{\bmH}^2\,.
\end{equation}
Interestingly, the unique solution to the above problem can be expressed in different ways.
From one hand, we can derive from the representer theorem in vv-RKHSs \citep{micchelli2005learning} the following expression:
\begin{equation}\label{eq:ridge_est_alpha}
    \hat{h}(x) = \sum_{i=1}^n \hat{\alpha}_i(x) \psiy(y_i), 
\end{equation}
with $\hat{\alpha}(x)=(\GramX+n \lambda I_n)^{-1} \kernelvectX \coloneqq \widehat{\Omega} \kernelvectX$, where $\GramX=\left(\kernelx\left(x_i, x_j\right)\right)_{i, j=1}^n$ and $\kernelvectX = \big(\kernelx(x, x_1), \ldots,\kernelx(x, x_n)\big)$.
On the other hand, using an operator view one obtains
\begin{equation}\label{eq:ridge_est_H}
    \hat{h}(x) =  \widehat{H} \psix(x)\,, 
    \end{equation}
where $\widehat{H}=\SY^\# \SX \big(\empcovx + \lambda I\big)^{-1}$.
The latter expression can be seen as a re-writing of the first \citep{Ciliberto2016}, echoing the KDE equations with finite-dimensional feature maps \citep{cortes2005}.
It can also be related to the conditional kernel empirical mean embedding \citep{grunewalder2012}.

%
%
%
%
%
%

The final estimator $\hat{f}$ is computed using the expression in \eqref{eq:ridge_est_alpha}, in order to benefit from the kernel trick:
\begin{equation}\label{eq:hatf}
\hat{f}(x) = \underset{y \in \bmY}{\arg \min} ~ \kernely(y, y) - 2 {\kernelvectX}^T \widehat{\Omega} {\kernelvectY}\,,
\end{equation}
where ${\kernelvectY} = \left(\kernely\left(y, y_1\right), \ldots, \kernely\left(y, y_n\right)\right)^\top$.
The training phase thus involves the inversion of a $n \times n$ matrix, whose cost without any approximation is $\bmO(n^3)$.
Besides, it implies storing $n^2$ values in memory, which induces a heavy space complexity as well.
In practice, decoding is performed by searching in a candidate set $\bmY_c \subseteq \bmY$ of size $n_c$. 
Hence, performing predictions on a test set $X_{\textnormal{te}}$ of size $n_{\textnormal{te}}$ mainly implies computing
\begin{equation}\label{eq:IOKR_pred}
    \underbrace{\GramX^{\textnormal{te}, \textnormal{tr}}}_{n_{\textnormal{te}} \times n} \, \underbrace{\widehat{\Omega}}_{n \times n}  \, \underbrace{\GramY^{\textnormal{tr}, c}}_{n \times n_c}\,,
\end{equation}
where $\GramX^{\textnormal{te}, \textnormal{tr}} = \left(\kernelx(x_i^{\textnormal{te}}, x_j)\right)_{1 \leq i \leq n_{\textnormal{te}}, 1 \leq j \leq n} \in \reals^{n_{\textnormal{te}} \times n}$, and $\GramY^{\textnormal{tr}, c} = \left(\kernely(y_i, y_j^c)\right)_{1 \leq i \leq n, 1 \leq j \leq n_c} \in \reals^{n \times n_c}$. The complexity of the decoding part is $\bmO\left(n_{\textnormal{te}} n n_c\right)$, considering $n_{\textnormal{te}} < n \leq n_c$.
IOKR thus suffers from both heavy time and space computational costs.
%
To cope with this limitation, we develop a general sketching approach that applies to both input and output feature spaces, accelerating both training and decoding.

\begin{figure*}[!t]
\centering
\subfigure{\label{fig:IOKR}\includegraphics[width=0.4\textwidth]{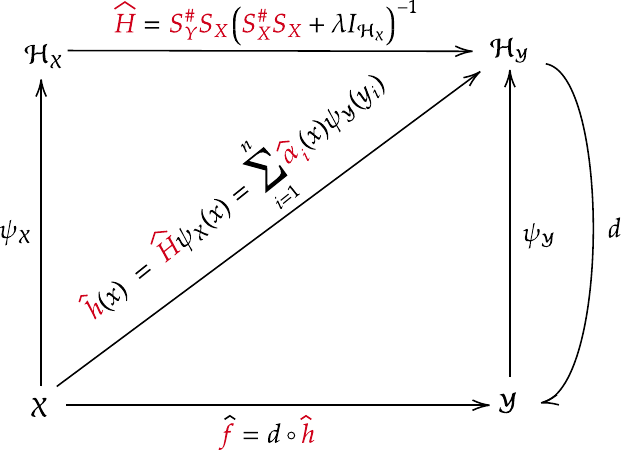}}
\qquad\quad
\subfigure{\label{fig:SISOKR}\includegraphics[width=0.4\textwidth]{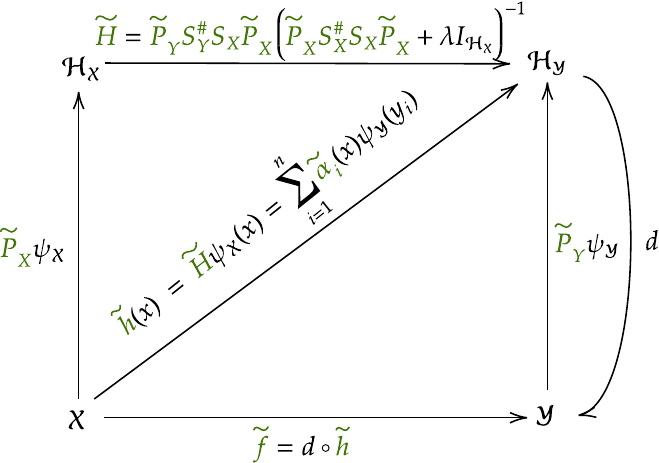}}
\caption{IOKR (left) and SISOKR (right) in the KDE setting.
Note that SISOKR consists in IOKR when kernels $\kernelz$ are replaced with their projected versions $\tilde{k}_\bmZ(\cdot, \cdot) = \langle \psiz(\cdot), \projZ \psiz(\cdot) \rangle_{\Hz}$.
%
However, this new output kernel changes the pre-image problem, and consequently the estimator $\tilde{f}$.
In the paper, we modify $\widetilde{H}$ (and not the kernels) in order to use the comparison inequality from \citet{ciliberto2020general},
see the proof of \cref{corollary:SISOKR_lr}.}
\label{fig:IOKR_SISOKR}
\end{figure*}

\paragraph{Sketching for kernel methods.} 
%
Applied to kernel methods to reduce their dependency in $n$, sketching can be seen as linear projections induced by a random matrix $R$ (the sketching matrix) drawn from a probability distribution over $\reals^{m \times n}$, where $m \ll n$.
Classic examples include Nystr\"{o}m's approximation, where each row of $R$ is randomly drawn from the rows of the identity matrix $I_n$, and Gaussian sketches, where all entries of $R$ are i.i.d. Gaussian random variables.
%
%
Nystr\"{o}m's approximation acts as a random training data sub-sampler, but it can be interpreted in many ways.
In \citet{drineas2005, bach2013sharp}, it is shown to generate a low-rank approximation of the Gram matrix, while in \citet{WilliamsNystromNIPS2000, Yang_NIPS2012_621bf66d}, it is seen as a way to construct data-dependent finite-dimensional random features. In \citet{rudi2015less}, instead, it is presented as a projection onto a small subspace of the RKHS.
%
%
For other sketching schemes such as Gaussian or Randomized Orthogonal Systems, most of the works adopt an optimization viewpoint, where a variable substitution is operated after the application of a Representer theorem \citep{Yang2017, lacotte2020adaptive}.
An interesting view provided in \citet{Kpotufe2020} explores the construction of random features based on Gaussian sketching.
%
All these works are however limited to sketching the \emph{input} kernel, in scalar regression problems.
%
In this work: (1) we generalize input sketching to vector-valued problems, (2) we sketch the outputs, which is critical to scale-up surrogate methods with kernelized outputs.

\section{SKETCHED INPUT SKETCHED OUTPUT KERNEL REGRESSION}\label{sec:SISOKR}

The goal of this section is to construct a low-rank estimator of $\hat{h}$ by using sketching on both the input and output kernels.
%
%
Note that sketching the feature maps is not desirable here: if we replace the output features $\psiy(y_i) \in \Hy$ with some sketch-dependent approximations $\tilde{\psi}_\bmY(y_i) \in \reals^m$ we become unable to compare the resulting $\tilde{h}$ to the target $h^*$.
Indeed, $\tilde{h}$ is an approximation of $x \mapsto \E_{y}[\tilde{\psi}_\bmY(y)|x]$, which is a biased version of $h^*$ due to the sketch realization.
Instead, as we show below, seeing sketching as orthogonal projections provides a natural way to solve our problem.
%
%
%
Ultimately, this gives rise to an estimator $\tilde{f}$ for structured prediction which is versatile, easy-to-implement, theoretically-based and scalable to large data sets.

\paragraph{Low-rank estimator.} Given two orthogonal projection operators $\projX$ and $\projY$, we start from \eqref{eq:ridge_est_H} and replace the sampling operators on both sides, $\SX$ and $\SY$, by their projected counterparts, $\SX \projX$ and $\SY \projY$, so as to encode dimension reduction.
The proposed low-rank estimator is expressed as follows:
\begin{equation*}
    \tilde{h}(x) = \projY \SY^{\#}\SX \projX\Big(\projX \empcovx \projX + \lambda I_{\Hx}\Big)^{-1} \psix(x)\,.
\end{equation*}
%
%
%
%
%
We now show how to design the projection operators using sketching and then derive the novel expression of the low-rank estimator in terms of a weighted combination of the training outputs: $\tilde{h}(x) = \sum_{i=1}^n \tilde{\alpha}_i \psiy(y_i)$, yielding a reduced computational cost.
IOKR and SISOKR approaches are illustrated on \cref{fig:IOKR_SISOKR}.
\paragraph{Sketching.} In this work, we chose to leverage sketching to obtain random projectors within the input and output feature spaces.
Indeed, sketching consists of approximating a feature map $\psiz : \bmZ \to \Hz$ by projecting it thanks to a random projection operator $\projZ$ defined as follows.
Given a random matrix $\sketchz \in \reals^{\mz \times n}$, $n$ data $(z_i)_{i=1}^n \in \bmZ$ and $\mz \ll n$, the linear subspace defining $\projZ$ is constructed as the linear subspace generated by the span of the following $\mz$ random vectors
\[
    \sum_{j=1}^n (\sketchz)_{ij} \psiz(z_j) \in \Hz, \quad i=1, \dots, \mz\,.
\]
%
%
%
One can show (\cref{prop:proj_exp} in \cref{apx:preliminary_res}) that the corresponding orthogonal projector writes 
\begin{equation}\label{eq:proj_core}
\projZ = (\sketchz \SZ)^{\#} \left(\sketchz \SZ (\sketchz \SZ)^{\#}\right)^{\dagger} \sketchz \SZ\,.    
\end{equation}
    %
    %
    %
    %

\begin{table*}[!ht]
\caption{Time and space complexities at training and inference for the IOKR and SISOKR algorithms with sub-sampling, $p$-sparsified ($p \in (0, 1]$) or Gaussian sketching, for a test set of size $n_{te}$  and a candidate set of size $n_c$, such that $n_{te} \leq \mx, \my < n \leq n_c$. For the sake of simplicity, we omit the $\bmO(\cdot)$ in the following.}
\begin{adjustbox}{center}
\begin{small}
\begin{tabular}{c|cc|cc}
    \toprule
    & \multicolumn{2}{c|}{Training} & \multicolumn{2}{c}{Inference} \\ 
    Method & Time & Space & Time & Space \\
    \midrule
    IOKR & $n^3$ & $n^2$ & $n_{te} n n_c$ & $n n_c$  \\
    SISOKR (sub-sampling) & $\max(\mx, \my) n$ & $\max(\mx, \my) n$ & $n_{te} \my n_c$ & $\my n_c$ \\
    SISOKR ($p$-sparsified) & $\max(\mx, \my)^2 p n$ & $\max(\mx, \my) p n$ & $\max(n_{te}, n \my p) \my n_c$ & $n p \my n_c$ \\
    SISOKR (Gaussian) & $\max(\mx, \my) n^2$ & $n^2$ & $n \my n_c$ & $n n_c$ \\
    \bottomrule
\end{tabular}
\end{small}
\end{adjustbox}
\label{table:complexities}
\end{table*}

\paragraph{Sketched Input Sketched Output Kernel Regression (SISOKR).}
The SISOKR estimator is the low-rank estimator $\tilde{h}$, where both $\projX$ and $\projY$ have been chosen as \eqref{eq:proj_core}, for some random sketches $\sketchx$ and $\sketchy$.
%
It also admits the following expression based on a linear combination of the $\psiy(y_i)$.
The proof of the following proposition is given in \cref{apx:preliminary_res}.
\begin{restatable}[Expression of SISOKR]{proposition}{propsisokr}\label{prop:SISOKR_exp}
$\forall\,x \in \bmX$,
\[
\tilde{h}\left(x\right)=\sum_{i=1}^n \tilde{\alpha}_i\left(x\right) \psiy\left(y_i\right)\,,
\]
where $\tilde{\alpha}\left(x\right) = \sketchy^\top \widetilde{\Omega} \sketchx \kernelvectX$ and
\[
\widetilde{\Omega} = \sketchGramY^\dagger \sketchy \GramY \GramX \sketchx^\top (\sketchx \GramX^2 \sketchx^\top + n \lambda \sketchGramX)^{\dagger}\,,
\]
with $\sketchGramX = \sketchx \GramX \sketchx^\top$ and \,$\sketchGramY = \sketchy \GramY \sketchy^\top$.
\end{restatable}
Note that the matrix quantity that we recover above, $\GramX \sketchx^\top \big(\sketchx \GramX^2 \sketchx^\top + n \lambda \sketchx \GramX \sketchx^\top\big)^{\dagger} \sketchx \kernelvectX$, is typical to sketched kernel Ridge regression \citep{rudi2015less, Yang2017}.
It allows to reduce the size of the matrix to invert, which is now an $\mx \times \mx$ matrix.
This is the main reason for the reduction of the learning step's complexity and is due to the input sketching.
Nonetheless, we still need to perform matrix multiplication $\sketchx \GramX$, whose efficiency depends on the sketch used).
Note that output sketching also requires additional operations, but the overall cost of computing $\tilde{\alpha}$ remains negligible compared to $\bmO(n^3)$, see ``training time'' column in \cref{table:complexities}.
As an example, with input/output Gaussian sketching which is the less efficient one, the time complexity is of order $\max(\mx,\my)n^2$, where $\mx,\my \ll n$.
We obtain the corresponding structured prediction estimator $\tilde{f}$ by decoding $\tilde{h}$, i.e., by replacing $\widehat{\Omega}$ by $\widetilde{\Omega}$ in \eqref{eq:hatf}.
In fact, the main quantity we have to compute for prediction is now
\begin{equation}\label{eq:SISOKR_pred}
    \underbrace{\GramX^{\textnormal{te}, \textnormal{tr}} \sketchx^\top}_{n_{\textnormal{te}} \times \mx} \underbrace{\widetilde{\Omega}}_{\mx \times \my} \underbrace{\sketchy \GramY^{\textnormal{tr}, c}}_{\my \times n_c}\,.
\end{equation}
The time complexity of this operation is $\bmO(n_{\textnormal{te}} m_\bmY n_c)$ if $n_{\textnormal{te}} \leq \mx, \my < n \leq n_c$, which is a significant complexity reduction (the dependence in $n$ vanishes), governed by the output sketch size $\my$, see \cref{table:complexities} for more details.

\section{THEORETICAL ANALYSIS}
\label{sec:theory}
In this section, we present a statistical analysis of the proposed estimators $\tilde h$ and $\tilde f$.
After introducing the assumptions on the learning task, we upper bound the excess-risk of the sketched kernel ridge estimator, highlighting the approximation errors due to sketching.
We then provide bounds for these approximation error terms.
Finally, we study under which setting the proposed estimators $\tilde h$ and $\tilde f$ obtain substantial computational gains, while still benefiting from a close-to-optimal learning rates.
We consider the following set of common assumptions in the kernel literature \citep{bauer2007, steinwart2009optimal, rudi2015less, pillaud2018statistical, fischer2020sobolev, ciliberto2020general, brogat-mottejmlr22}.

\begin{assumption}[Attainability]\label{asm:attainable_case}
We assume that $h^* \in \bmH$, i.e., that there is a linear operator $H : \Hx \rightarrow \Hy$, with $\left\|H\right\|_{\HS} < + \infty$, s.t. $h^*(x) = H \psix(x)$, $\forall\, x \in \bmX$.
\end{assumption}

This is a standard assumption in the context of least-squares regression \citep{caponnetto2007optimal}, making the target $h^*$ belong to the hypothesis space.
Note that relaxing this assumption is possible, although it would add a bias term that still requires some knowledge about $h^*$ to be bounded.
For instance, if $h^*$ is supposed to be square-integrable, one usually chooses a RKHS associated with a universal operator-valued kernel, which is dense in the space of the square-integrable functions \citep[Section~4]{carmeli2010vector}.
We now describe a set of generic assumptions that have to be satisfied by both input and output kernels $\kernelx$ and $\kernely$.

\begin{assumption}[Bounded kernel]\label{asm:bounded_input_output_kernels} 
There exists $\kappaz > 0$ such that $\kernelz(z,z) \leq \kappaz^2$,  $\forall\, z \in \bmZ$. We note $\kappax, \kappay > 0$ for the input and output kernels $\kernelx$ and $\kernely$ respectively.
\end{assumption}

\begin{assumption}[Capacity condition]\label{asm:capacity} There exists $\gammaz \in [0,1]$ such that $ \Qz := \Tr(\covz^{\gammaz}) < +\infty$.
\end{assumption}

Note that \Cref{asm:capacity} is always verified for $\gammaz = 1$, as $\Tr(\covz) = \E[\|\psiz(z)\|_{\Hz}^2] < + \infty $ from \cref{asm:bounded_input_output_kernels}, and that the smaller $\gammaz$ the faster the eigendecay of $\covz$, with $\gammaz = 0$ when $\covz$ is of finite rank.
More generally, this assumption is for instance verified for a Sobolev kernel and a marginal distribution whose density is upper-bounded \citep[Assumption~2]{ciliberto2020general}.

\begin{assumption}[Embedding property]\label{asm:emb} There exist $\bz >0$ and $\muz \in [0,1]$ such that $\psiz(z) \otimes \psiz(z) \preceq \bz \covz^{1-\muz}$ almost surely.
\end{assumption}

Note that \Cref{asm:emb} is always verified for $\muz = 1$, as $\psiz(z) \otimes \psiz(z) \preceq \kappaz^2 I_{\Hz}$ by \cref{asm:bounded_input_output_kernels}, and that the smaller $\muz$, the stronger the assumption, with $\muz=0$ when $\covz$ is of finite.
It allows to control the regularity of the functions in $\Hz$ with respect to the $L^{\infty}$-norm, as it implies $\|h\|_{L_\infty} \leq \bz^{1/2} \|h\|_{\Hz}^{\mu} \E[h(z)^2]^{(1-\mu)/2}$ \citep{pillaud2018statistical}.
For instance, an absolutely continuous distribution whose density is lower-bounded almost everywhere and a Mat\'{e}rn kernel verifies \Cref{asm:emb} \citep[Example~2]{pillaud2018statistical}.

\paragraph{SISOKR Excess-Risk.}
We can now provide a bound on the excess-risk of SISOKR.

\begin{restatable}[SISOKR excess-risk bound]{theorem}{thmsisokr}\label{th:sketched_ridge}
Let $\delta \in (0,1]$, $n \in \mathbb{N}$ such that $\lambda = n^{-1/(1+\gammax)} \geq \frac{9\kappax^2}{n} \log(\frac{n}{\delta})$.
Under \Cref{asm:attainable_case,asm:bounded_input_output_kernels,asm:capacity,asm:emb}, with probability $1-\delta$ we have
\begin{align}
\E_x\Big[\|\tilde h(x) &- h^*(x)\|^{2}_{\Hy}\Big]^{\frac{1}{2}}\nonumber\\
&\leq  S(n, \delta) + c_2 A_{\rhox}^{\psix}(\widetilde P_X) + A_{\rhoy}^{\psiy}(\widetilde P_Y)\,,\label{eq:decompo_excess_risk}
\end{align}
where $S(n, \delta) = c_1 \log(4/\delta)\,n^{-\frac{1}{2(1 + \gammax)}}$ and
\[
A_{\rhoz}^{\psiz}(\widetilde P_Z) = \E_z\Big[\|(\projZ -I_{\Hz})\psiz(z)\|_{\Hz}^2\Big]^{\frac{1}{2}}\,,
\]
with $c_1, c_2 >0$ constants independent of $n$ and $\delta$.
\end{restatable}

\begin{proof}[Proof sketch.]
The proof relies on a decomposition of the operator $\widetilde{H}$ such that $\tilde{h}(x) = \widetilde{H}\psi_\mathcal{X}(x)$, see \eqref{eq:decompo_SISOKR}.
The first term in \eqref{eq:decompo_excess_risk} corresponds to the non-sketched kernel Ridge regression error, and the second term to the input sketching error.
The latter extends both the results of \citet{ciliberto2020general} to sketched estimators, and that of \citet{rudi2015less} to the vector vector-valued case.
The third term, i.e., the output sketching error is specific to our framework and derives from the expression of $h^*$ and Jensen's inequality.
\end{proof}

The learning rate of the first term, i.e., the non-sketched kernel Ridge regression error, has been shown to be optimal under our set of assumptions in a minimax sense \citep{caponnetto2007optimal}.
The second and the third terms are approximation errors due to the sketching of the input and the output kernels, respectively.
In particular, they write as \textit{reconstruction errors} \citep{blanchard2007statistical} associated to the random projection $\projX$ and $\projY$ of the feature maps $\psix$ and $\psiy$ through the input and output marginal distributions.


\paragraph{Sketching Reconstruction Error.}
In \cref{th:sketching}, we give bounds on the sketching reconstruction error for the family of sub-Gaussian sketches, enlarging the scope of sketching distributions whose reconstruction error's bound is known ---it was previously limited to uniform and approximate leverage scores sub-sampling sketches \citep{rudi2015less}.
More generally, note that are admissible in our theoretical framework all sketching distributions for which concentration bounds on the induced empirical covariance operators can be derived, since quantity $A_{\rhoz}^{\psiz}(\widetilde P_Z)$ is then easily controlled.
We now recall the definition of sub-Gaussian sketches, and show how to bound their reconstruction error.

\begin{definition}\label{def:subG_sketch}
    A sub-Gaussian sketch $\sketchz \in \reals^{\mz \times n}$ is composed of i.i.d. entries such that $\E\left[\sketchzij\right] = 0$, $\E\left[\sketchzij^2\right] = 1/\mz$ and $\sketchzij$ is $\frac{\nuz^2}{\mz}$-sub-Gaussian, for all $1 \leq i \leq \mz$ and $1 \leq j \leq n$, where $\nuz \geq 1$.
\end{definition}

Recall that a standard normal r.v. is $1$-sub-Gaussian.
Moreover, by Hoeffding's lemma, any r.v. taking values in a bounded interval $[a, b]$ is $(b-a)^2/4$-sub-Gaussian.
Hence, any sketch matrix composed of i.i.d. Gaussian or bounded r.v. is a sub-Gaussian sketch.
Finally, note that $p$-sparsified sketches \citep{elahmad2023fast} are sub-Gaussian with $\nuz^2 = 1/p$, with $p \in ]0,1]$.

\begin{restatable}[sub-Gaussian sketching reconstruction error]{theorem}{thmsketching}\label{th:sketching}  For $\delta \in \left(0, 1/e\right]$, $n \in \mathbb{N}$ sufficiently large such that $\frac{9}{n} \log(n / \delta) \leq n^{-\frac{1}{1+\gammaz}} \leq \|\covz\|_{\op}/2$, then if 
\begin{align}\label{eq:m_lower_bound}
    \mz \geq c_4 \max\left(\nu_\mathcal{Z}^2\,n^{\frac{\gammaz + \muz}{1+\gammaz}}, \nu_\mathcal{Z}^4 \log\left(1/\delta\right)\right)\,,
\end{align}
with probability $1-\delta$ we have
\begin{align}\label{eq:sketch_rec_error}
     \E_z\Big[\|(\projZ -I_{\Hz})\psiz(z)\|_{\Hz}^2\Big] &\leq c_3\,n^{-\frac{1-\gammaz}{1+\gammaz}}\,,
\end{align}
where $c_3, c_4 >0$ are constants independents of $n, \mz, \delta$.
\end{restatable}

\begin{proof}[Proof sketch]
The proof essentially consists in bounding the difference between the empirical covariance operator and its sketched counterpart in operator norm, see \eqref{eq:diff_op_norms}. The latter rewrites as a sum of sub-Gaussian random variables in a separable Hilbert space, and we invoke \citet[Theorem~9]{Koltchinskii_2017}.
\end{proof}

Hence, depending on the regularity of the distribution (defined through our set of assumptions), one can obtain a small reconstruction error even with a small sketching size.
For instance, if $\muz = \gammaz = 1/3$, one obtains a reconstruction error of order $n^{-1/2}$ by using a sketching size of order $n^{1/2} \ll n$.
As a limiting case, when $\muz = \gammaz = 0$, one obtains a reconstruction error of order $n^{-1}$ when using a constant sketching size.
\medskip

\begin{remark}[Comparison to Nystr\"{o}m's approximation]
Note that the rate in \Cref{th:sketching} is the same as that obtained with Nystr\"{o}m's approximation.
However, our lower bound on the sketching size is slightly better.
Recall that for uniform Nystr\"{o}m it is of order $\max\Big(n^{\frac{\gammaz + \muz}{1+\gammaz}}, 1\Big) \left(\log(n) + \log\left(4 \kappaz^2/\delta\right)\right)$.
\end{remark}

\begin{remark}[Relaxation of \cref{asm:emb}]
\cref{asm:emb} allows to derive an upper bound of $\bmN_{\bmZ}^{\infty}(t)$, with $t = n^{-\frac{1}{1+\gammaz}}$, that appears in the lower bound of the sketching size $\mz$, see \cref{lem:emb} in \cref{apx:aux_res} and the proof of \cref{th:sketching} in \cref{apx:th2}.
However, we also have that $\bmN_{\bmZ}^{\infty}(t) \leq t^{-1}$, hence, if $\muz + \gammaz \geq 1 + \frac{\log(\bz \Qz)(1+\gammaz)}{\log(n)}$, we can relax \cref{asm:emb} and rather obtain
\begin{equation}
    \mz \geq c_4 \max\left(\nu_\mathcal{Z}^2\,n^{\frac{1}{1+\gammaz}}, \nu_\mathcal{Z}^4 \log\left(1/\delta\right)\right)\,,
\end{equation}
as a lower bound.
\end{remark}

\paragraph{Learning rates for SISOKR with sub-Gaussian sketches.} For the sake of presentation, we use $\lesssim$ to keep only the dependencies in $n, \delta, \nu, \gamma, \mu$. We note $a \lor b \coloneqq \max(a,b)$.

\begin{corollary}[SISOKR learning rates]\label{corollary:SISOKR_lr}
Consider the Assumptions of \cref{th:sketched_ridge,th:sketching}, that $\|\psiy(y)\|_{\Hy} = \kappay$ for all $ y \in \bmY$, and $n \in \mathbb{N}$ such that $\frac{9}{n} \log(n / \delta) \leq n^{-\frac{1}{1+\gamma_\mathcal{Z}}} \leq \|\textnormal{C}_\mathcal{Z}\|_{\op}/2$ for $\mathcal{Z} \in \{\mathcal{X}, \mathcal{Y}\}$.
Set
\begin{equation}
    \mz \gtrsim \max\left(\nu_\mathcal{Z}^2\,n^{\frac{\gammaz + \muz}{1+\gammaz}}, \nu_\mathcal{Z}^4 \log\left(1/\delta\right)\right)
\end{equation}
for $\mathcal{Z} \in \{\mathcal{X}, \mathcal{Y}\}$. Then with probability $1-\delta$
\begin{align}\label{eq:ex_ri_f}
    \bmR(\tilde f) - \bmR(f^*) &\lesssim \log\left(4/\delta\right) n^{-\frac{1-\gammax \lor \gammay}{2(1+\gammax \lor \gammay)}}\,.
\end{align}
\end{corollary}


\begin{proof}
Using \Cref{th:sketched_ridge,th:sketching} to bound $A_{\rhox}^{\psix}(\widetilde P_X)$ and $A_{\rhoy}^{\psiy}(\widetilde P_Y)$ gives that with probability $1 -\delta$ it holds $\E_x\big[\|\tilde h(x) - h^*(x)\|^{2}_{\Hy}\big]^{\frac{1}{2}} \lesssim \log\left(4/\delta\right) n^{-\frac{1-\gammax \lor \gammay}{2(1+\gammax \lor \gammay)}}$.
We then apply the comparison inequality \citep{ciliberto2020general} to the loss $\Delta(y,y') = \|\psiy(y) - \psiy(y')\|^2_{\Hy}$.
\end{proof}

This corollary shows that under strong enough regularity assumptions, the proposed estimators benefit from a close-to-optimal learning rate, even with small input and output sketching sizes. For instance, if $\mux = \muy = \gammax = \gammay = 1/3$, one obtains a learning rate of $\bmO(n^{-1/4})$, instead of the optimal rate of $\bmO(n^{-3/8})$ under the same assumptions, but only requiring sketching sizes $\mx, \my$ of order $n^{1/2} \ll n$. As a limiting case, when $\mux = \muy = \gammax = \gammay = 0$, one attains the optimal $\bmO(n^{-1/2})$ learning rate using constant sketching sizes.


\begin{remark}[Other Sketches]
Although we focused on sub-Gaussian sketches, any sketching distribution admitting concentration bounds for operators on separable Hilbert spaces allows to bound the quantity $A_{\rhoz}^{\psiz}(\widetilde P_Z)$ and is then admissible for our theoretical framework. For instance, as showed in \cite{rudi2015less}, uniform and approximate leverage scores sub-sampling schemes fit into the presented theory.
\end{remark}


%
%
%
%
%

\section{EXPERIMENTS}
\label{sec:expes}

\begin{figure*}[!t]
\centering
%
\includegraphics[width=0.288\textwidth]{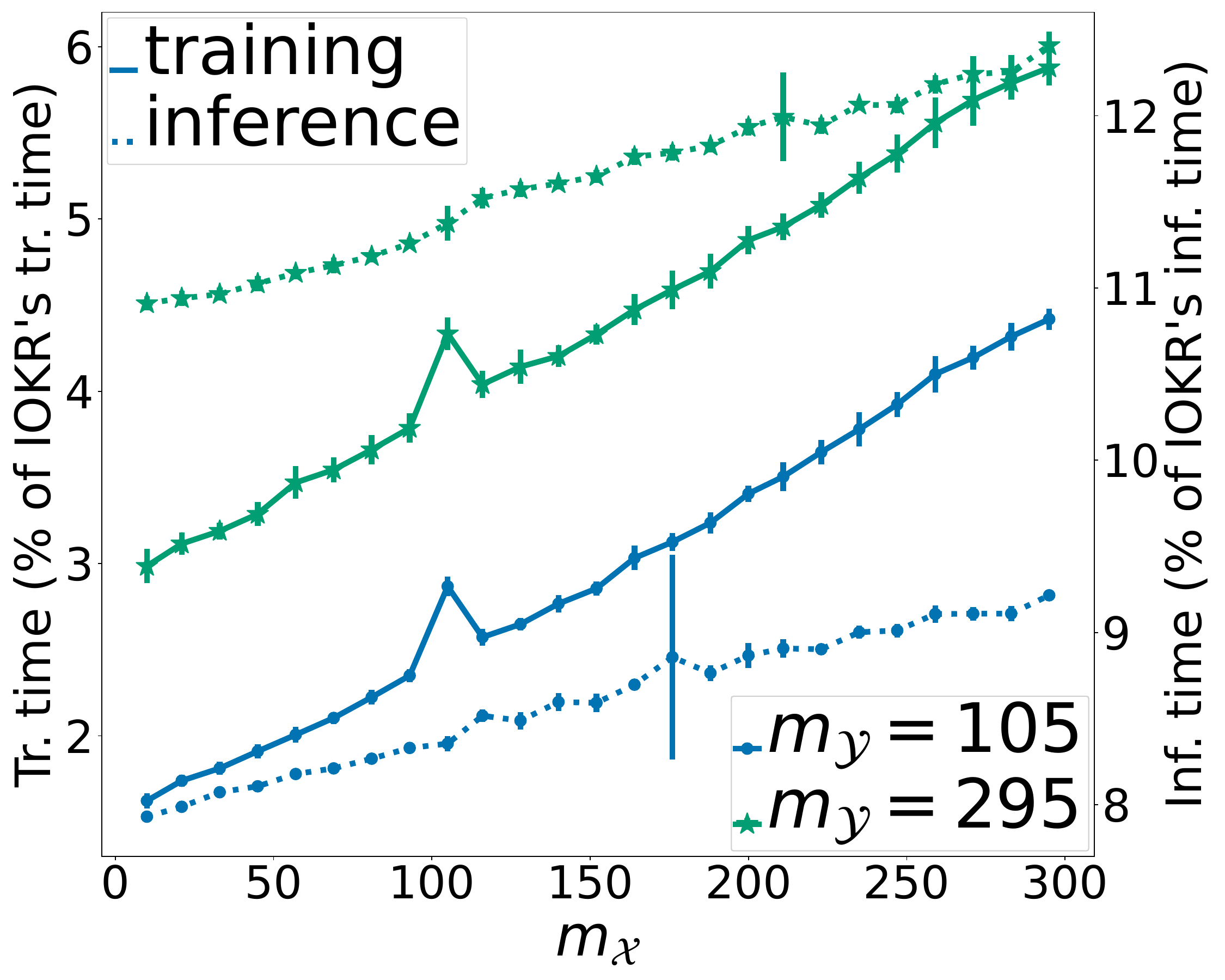}
\hfill
\includegraphics[width=0.288\textwidth]{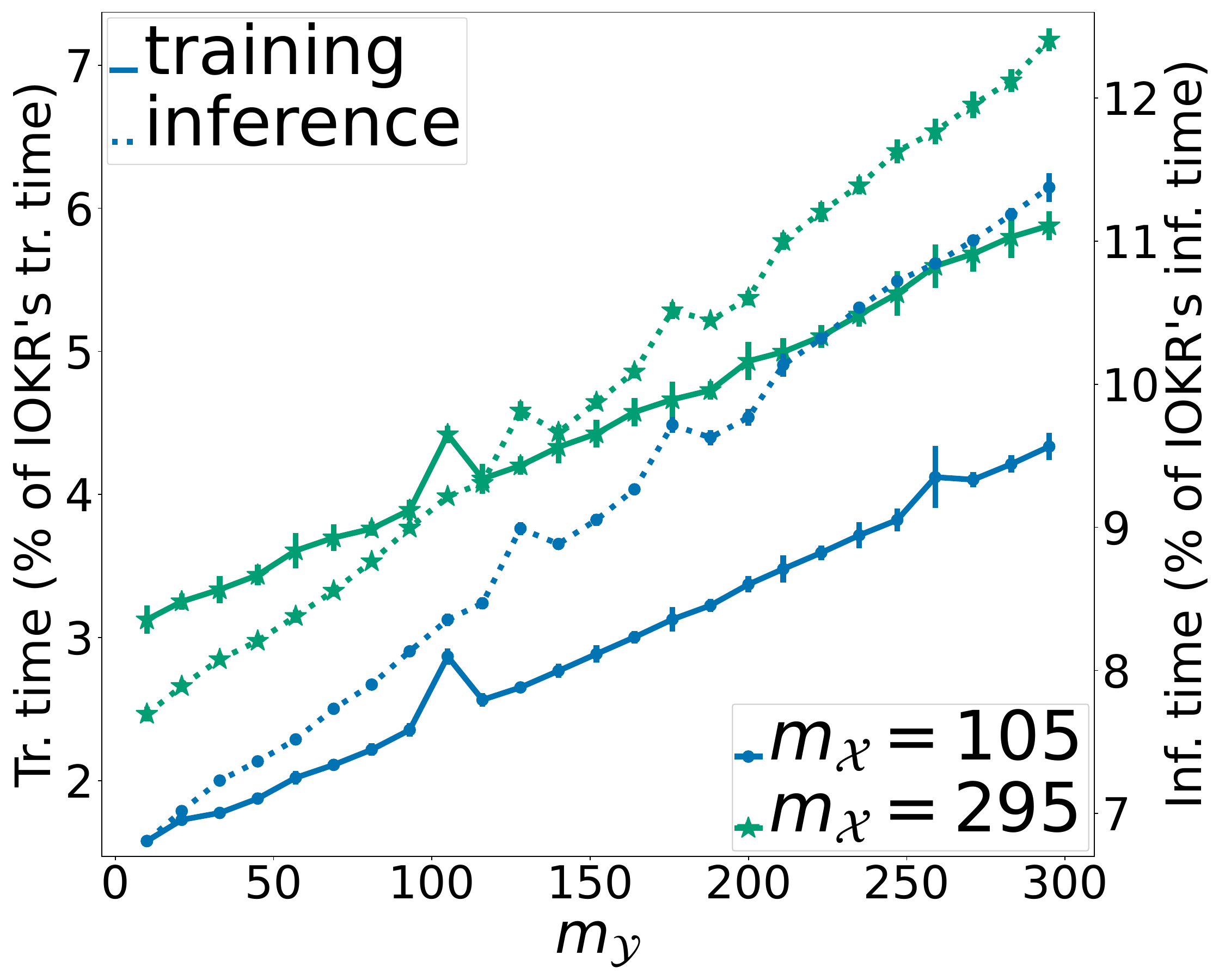}
\hfill
\includegraphics[width=0.288\textwidth]{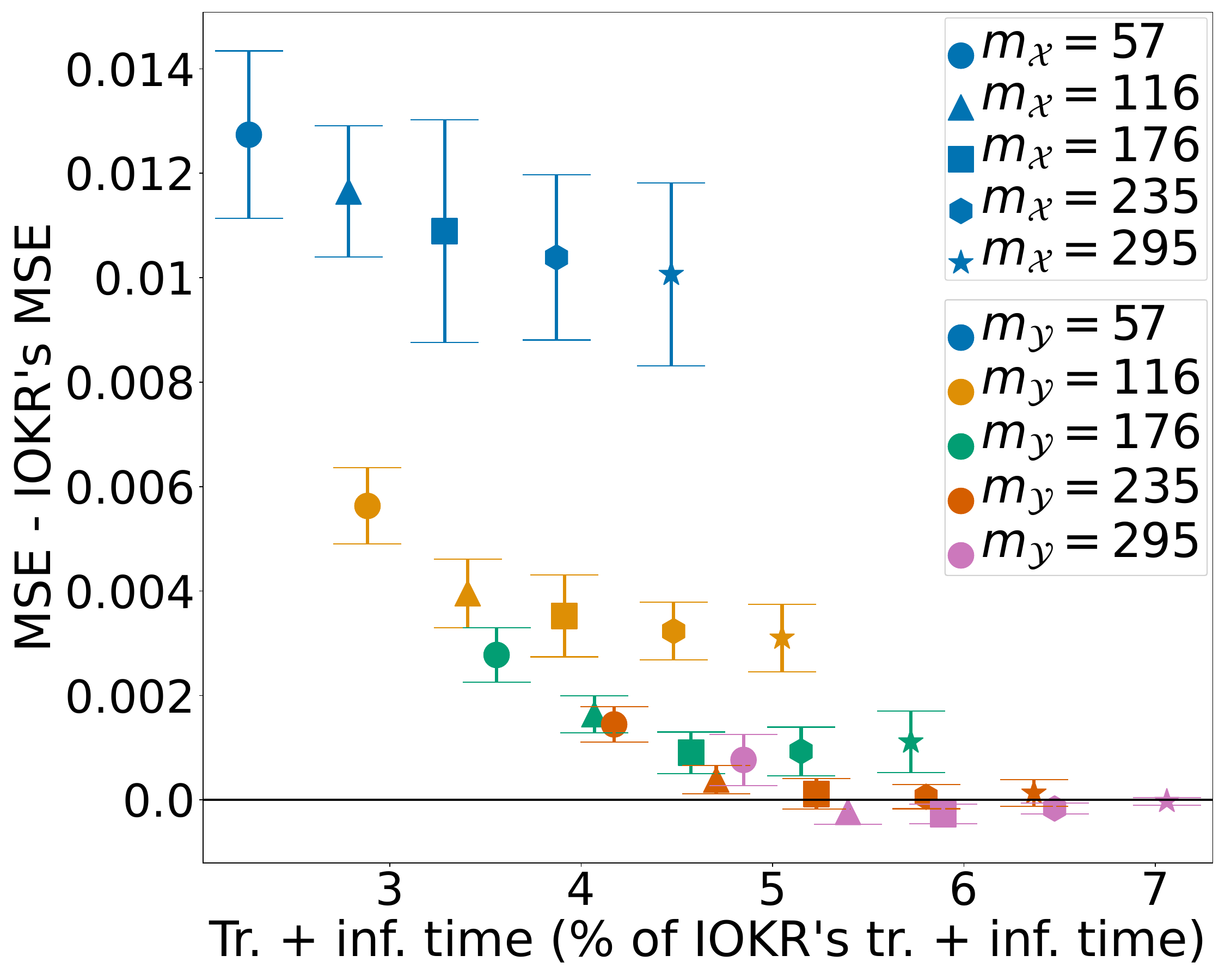}
\caption{Variation of training and inference time w.r.t. $\mx$ and $\my$ (left and center), and trade-off performance against computational time (right) for SISOKR with $(2 \cdot 10^{-3})$-SR input/output sketches on synthetic data.}
\label{fig:toy_figures}
\end{figure*}

In this section, we present experiments on synthetic and real-world data sets.
SIOKR and ISOKR denote the models with sketching leveraged only on the inputs (resp. outputs).
Results are averaged over 30 replicates, unless for the metabolite's experiments (5 replicates).
%

\paragraph{On the choice of the sketching types and its hyper-parameters.} We focus on uniform sub-sampling \citep{rudi2015less} and $p$-sparsified ($p$-SR/SG) \citep{elahmad2023fast} sketches, which are covered by our theory.
Sub-sampling is the most efficient approach computationally, but we empirically observe that $p$-SR/SG sketching is more accurate statistically.
For SIOKR/ISOKR, we privilege accuracy and $p$-SR/SG sketching, as it is already providing substantial training/inference accelerations.
Regarding SISOKR, we want the method to be the fastest both in training and inference.
However, since output sketching adds training computations, we compensate and use input sub-sampling to remain faster in training than SIOKR.
Regarding the input/output sketching sizes $\mx$ and $\my$, the first way consists of leveraging the theoretical lower bounds derived for $\mx$ and $\my$, see Equation~\eqref{eq:m_lower_bound}.
Indeed, by computing the Singular Value Decomposition of the input/output Gram matrix, one may determine their eigendecay (i.e., $\gamma_\mathcal{Z}$, $\mu_\mathcal{Z}$, $\nu_\mathcal{Z}$) and set $\mx$ and $\my$ accordingly.
However, computing the SVD is very expensive, hence one can rather compute the approximate leverage scores as in \citet{elalaoui_NIPS2015} for instance.
In the following, we instead adopt an empirical routine.
Given training and/or inference time budgets (corresponding e.g., to IOKR's training/inference times or the hardware limitations), we start from small $\mx$ and $\my$, which we progressively increase to maximize accuracy while respecting the budget.
For the $p$-SR/SG sketches, we always set $p = 20/n$.

\paragraph{Synthetic Least Squares Regression.} We generate a synthetic data set of least-squares regression, with $n=10\,000$ training data points, $\bmX = \bmY = \reals^d$, $d = 300$, and use input and output linear kernels, hence $\Hx = \Hy = \reals^d$.
We construct covariance matrices $\covx$ and $E$ by drawing randomly their eigenvectors such that their eigenvalues are $\sigma_k(\covx) = k^{-3/2}$ and $\sigma_k(E) = 0.2\,k^{-1/10}$.
We draw $H_0 \in \reals^{d \times d}$ with i.i.d. coefficients from the standard normal distribution and set $H = \covx H_0$.
For $i \leq n$, we generate inputs $x_i \sim \bmN(0, \covx)$, noise $\epsilon_i \sim \bmN(0, E)$ and outputs $y_i = H x_i + \epsilon_i$.
We generate validation and test sets of $n_\textnormal{val}=n_\textnormal{te}=1000$ points in the same way.
Such choices for $\covx$ (with a polynomial eigenvalue decay), $E$ (with very low eigenvalues and eigenvalue decay), and $H = \covx H_0$ enforce a high eigenvalue decay for $\covy$ (since it will have a similar eigendecay as $\covx$) while being a favorable setting to deploy sketching, as the true regression function $H$ is low rank.
%
%
We select the regularisation penalty $\lambda$ via 1-fold cross-validation.
We learn the SISOKR model for different values of $\mx$ and $\my$ (from 10 to 295) and $(2 \cdot 10^{-3})$-SR input and output sketches.
Note that for such a problem where $\bmY = \Hy$, no decoding step is needed for inference.
We still perform an artificial pre-image problem to illustrate the computational benefit of sketching during this phase.

\cref{fig:toy_figures} (left and center) presents computational training (solid lines) and inference (dotted lines) time (as a percentage of IOKR's training/inference time) w.r.t. $\mx$ (resp. $\my$) for two values of $\my$ (resp. $\mx$).
First, since $\mx, \my \leq 295 \ll n = 10\,000$, note that SISOKR's training and inference times are significantly smaller than IOKR's (between 2 and 6\% of IOKR's training time and 8 and 12\% IOKR's inference time).
On \Cref{fig:toy_figures} (left) the slopes of the training time's lines are higher than the inference time's ones, while the opposite happens on \Cref{fig:toy_figures} (center).
This confirms that training complexity is more sensitive to $\mx$, while inference complexity is governed by $\my$.
\Cref{fig:toy_figures} (right) presents the difference with IOKR's test errors, in terms of Mean Squared Error (MSE), for some choices of $\mx$ and $\my$, as a function of the sum of the training and inference times.
The MSE decreases as the sketch sizes increase and at a faster rate with respect to $\mx$.
This might be due to the fact that we directly control the eigendecay of $\covx$, whereas $\covy = \covx H_0 \covx H_0^\top \covx + E$, such that its range is not totally controlled by $\covx$.
SISOKR obtains better MSE performance than IOKR for $\mx \geq 116$ and $\my = 295$, which is consistent with the results obtained when applying sketching to the input (resp. output) kernel only, see \cref{apx:synthetic}.

\begin{table}[!t]
\centering
\captionsetup{singlelinecheck = false, justification=raggedright, font=normalsize}
\caption{$F_1$ scores on tag prediction from text data.}
\label{table:multi-label-F1scores}
\begin{scriptsize}
\begin{tabular}{ c c c c }
    \toprule
    Method & Bibtex & Bookmarks & Mediamill \\ 
    \hline
    LR & $37.2$ & $30.7$ & NA \\
    SPEN & $42.2$ & $34.4$ & NA \\
    PRLR & $44.2$ & $34.9$ & NA \\
    DVN & $44.7$ & $37.1$ & NA \\
    \hline
    SISOKR & $44.1 \pm 0.07$ & $\textbf{39.3} \pm 0.61$ & $57.26 \pm 0.04$ \\
    ISOKR & $44.8 \pm 0.01$ & NA & $58.02 \pm 0.01$ \\
    SIOKR & $44.7 \pm 0.09$ & $39.1 \pm 0.04$ & $57.33 \pm 0.04$  \\
    IOKR & $\textbf{44.9}$ & NA & $\textbf{58.17}$ \\
    \bottomrule
\end{tabular}
\end{scriptsize}
\end{table}

\paragraph{Multi-Label Classification.} We compare our models to state-of-the-art multi-label and structured prediction methods, namely IOKR \citep{brouard2016input}, logistic regression (LR) trained independently for each label \citep{Lin_multi-labellearning}, the multi-label approach Posterior-Regularized Low-Rank (PRLR) \citep{Lin_multi-labellearning}, the energy-based model Structured Prediction Energy Networks  (SPEN) \citep{belanger2016structured} and Deep Value Networks (DVN) \citep{Gygli2017}.
%
Results are taken from the cited articles.
Data sets Bibtex and Bookmarks are tag recommendation problems, in which the objective is to propose a relevant set of tags (e.g., url, description, journal volume) to users when they add a new Bookmark or Bibtex entry to the social bookmarking system Bibsonomy.
%
The MediaMill Challenge \citep{snoek_mediamill} is a multi-label classification problem, where the goal is to detect the presence of semantic concepts in a video.
They contain respectively $n = 4880$, $n = 60\,000$ and 
$n = 30\,993$ training points, see \cref{apx:multi-label} for details.
We use the train-test splits available at \url{https://mulan.sourceforge.net/datasets-mlc.html}.

\begin{table*}[!t]
\centering
\caption{Training/inference times (in seconds).}
\label{table:times_multi_label}
\begin{scriptsize}
\begin{tabular}{ c c c c }
    \toprule
    Method & Bibtex & Bookmarks & Mediamill \\ 
    \midrule
    SISOKR & $\textbf{1.41} \pm \textbf{0.03}$ / $\textbf{0.46} \pm \textbf{0.01}$ & $\textbf{118} \pm \textbf{1.5}$ / $\textbf{20} \pm \textbf{0.2}$ & $\textbf{66} \pm \textbf{0.1}$ / $\textbf{4} \pm \textbf{0.01}$ \\
    ISOKR & $2.51 \pm 0.06$ / $0.58 \pm 0.01$ & NA & $636 \pm 3.7$ \ $9 \pm 0.2$ \\
    SIOKR & $1.99 \pm 0.07$ / $1.22 \pm 0.03$ & $354 \pm 2.1$ / $297 \pm 2.1$ & $199 \pm 0.1$ / $121 \pm 0.02$ \\
    IOKR & $2.54$ / $1.18$ & NA & $621$ / $204$ \\
    \bottomrule
\end{tabular}
\end{scriptsize}
\end{table*}

\begin{table*}[!t]
\caption{Standard errors for the metabolite identification problem and computation times (in seconds).}
\begin{adjustbox}{center}
\begin{scriptsize}
\begin{tabular}{ c c c c c }
    \toprule
    Method & kernel loss & Top-$1$ | $5$ | $10$ accuracies & training & inference \\ 
    \midrule
    SPEN & $0.537 \pm 0.008$ & $25.9 \%$ | $54.1 \%$ | $64.3 \%$ & NA & NA \\
    \hline
    SISOKR & $0.566 \pm 0.007$ & $25.1 \%$ | $54.2 \%$ | $64.7 \%$ & $4.05 \pm 0.05$ & $\textbf{1112} \pm \textbf{29}$ \\
    ISOKR & $0.509 \pm 0.009$ & $28.0 \%$ | $58.9 \%$ | $68.9 \%$ & $6.25 \pm 50.31$ & $1133 \pm 32$ \\
    SIOKR & $0.492 \pm 0.008$ & $29.5 \%$ | $61.3 \%$ | $70.9 \%$ & $\textbf{1.25} \pm \textbf{0.02}$ & $1179 \pm 37$ \\
    IOKR & $\textbf{0.486} \pm \textbf{0.008}$ & $\textbf{29.6} \%$ | $\textbf{61.6} \%$ | $\textbf{71.4}\%$ & $3.54 \pm 0.15$ & $1191 \pm 38$ \\
    \bottomrule
\end{tabular}
\end{scriptsize}
\end{adjustbox}
\label{table:metabolites_results}
\end{table*}

For all multi-label experiments, we use Gaussian input and output kernels with widths $\sigma_\textnormal{in}^2$ and $\sigma_\textnormal{out}^2$.
We use $p$-SG input (resp. output) sketches for SIOKR (resp. ISOKR), uniform sub-sampling input sketches and $p$-SG output sketches for SISOKR.
For Bibtex experiments, we choose $\mx = 2250$ and $\my = 200$, for Bookmarks experiments, $\mx = 13\,000$ and $\my = 750$, and for Mediamill experiments, $\mx = 8\,000$ and $\my = 500$.
%
All the training data are used as candidate sets.
The performance is measured by example-based F1 score, and hyper-parameters are selected on logarithmic grids by 5-fold cross-validation.
The results in \cref{table:multi-label-F1scores} show that surrogate methods (last four columns) compete with SOTA methods, including deep-learning-based methods such as SPEN or DVN.
On Bibtex, sketched models preserve good performance compared to IOKR (which performs best) while being faster to train (SIOKR and SISOKR) and significantly faster for inference (ISOKR and SISOKR), see \cref{table:times_multi_label}.
Since the Bookmarks data set is too large, storing the whole $n^2$-Gram matrix $\GramX$ exceeds CPU's space limitations, which highlights the necessity of efficient sketching approximations such that sub-sampling or $p$-SR/SG sketches for kernel methods.
Hence, we can only test SIOKR and SISOKR models on this data set, which outperforms other methods.
SISOKR's inference phase is notably faster than SIOKR's (20 seconds vs. 5 minutes).
Similarly, on the Mediamill problem, our approximated approaches are shown to be significantly faster to run while suffering a minimal reduction in $F_1$ score.
Note that, with the same sketch matrix $\sketchx$, SIOKR's training is faster than SISOKR's as there is no additional computation on Gram matrix $\GramY$.
In \cref{table:times_multi_label}, SISOKR is faster to train as it uses a more efficient input sketching (sub-sampling vs. $p$-SG).

\paragraph{Metabolite Identification.}
Metabolite identification consists here of predicting small molecules, called metabolites, from their tandem mass spectrum.
%
%
The metabolite structure is represented as a binary vector of length $d = 7593$, called a fingerprint.
Each entry of the fingerprint encodes the presence or absence of a molecular property.
IOKR is the SOTA method for this problem \citep{Brouard_ismb2016}.
The data set consists of $n = 6974$ training labeled mass spectra, the median size of the candidate sets is $292$ and the largest candidate set contains $36\,918$ fingerprints.
This metabolite identification problem thus involves high-dimensional complex outputs, for which the choice of the output kernel is crucial, and a large number of candidates, making the inference step long.

Our experimental protocol is similar to that of \citet{Brouard_ismb2016} (5-CV Outer / 4-CV Inner loops).
We use probability product input kernel for mass spectra and Gaussian-Tanimoto output kernel \citep{RALAIVOLA20051093} -- with width $\sigma^2$ -- for the molecular fingerprints.
We select hyper-parameters $\lambda$ and $\sigma^2$ in logarithmic grids based on MSE in $\Hy$ (hence no decoding is needed during selection).
For the sketched models, we use $p$-SR input (resp. output) sketches for SIOKR (resp. ISOKR), and uniform sub-sampling input sketches and $p$-SR output sketches for SISOKR, with $\mx = 1500$, and $\my = 800$.

We compare our sketched models with IOKR and SPEN, see \cref{table:metabolites_results}.
Results for SPEN are taken from \citet{brogat-mottejmlr22}.
SIOKR obtains results similar to IOKR while being slightly faster in both the training and inference phases.
ISOKR is slightly less accurate, but outperforms (S)IOKR in terms of inference time, while SISOKR has the fastest inference phase and still competes with SPEN statistically.
We observe here that it is difficult to reduce significantly the inference time while keeping a good accuracy and to reduce both the training and inference time.
This is due to the particular setting of the metabolite data set.
Indeed, each molecule is associated with a specific candidate set, so when performing predictions one has to run through each element one by one to pick its candidate set.
When performing predictions, one has to compute the matrix multiplication \eqref{eq:SISOKR_pred}, which has a smaller complexity than \eqref{eq:IOKR_pred}, given that $\sketchy \GramY^{\textnormal{tr}, c}$ is already known.
However, in the case of metabolite identification, one has to perform it for each test data, which takes most of the inference for both ISOKR and SISOKR models.
As an example, for the 1133 (resp. 1112) seconds-long ISOKR's (resp. SISOKR) inference phase, computing $\sketchy \GramY^{\textnormal{tr}, c}$ takes 940 (resp. 917) seconds.
Since we have access to all candidate sets for each molecule, one could pre-process these data beforehand and perform these matrix multiplications during training, leading to a high training time, but a very small inference time, which could be of interest according to the practitioner's wish.
When candidate sets are known and fixed (e.g., in multi-label prediction), sketching the output kernel is of particular interest as no additional operation is needed for each prediction.

\section{CONCLUSION}
\label{sec:conclu}

In this paper, we scale-up surrogate methods for structured prediction based on kernel Ridge regression by using random projections for both inputs and outputs.
%
%
%
An interesting avenue for future work is the study of non-parametric estimators with kernelized outputs that do not benefit from the Ridge regression closed-form.

\subsubsection*{Acknowledgments}

This work was supported by the Télécom Paris research chair on Data Science and Artificial Intelligence for Digitalized Industry and Services (DSAIDIS) and the French National Research Agency (ANR) through ANR-18-CE23-0014 APi (Apprivoiser la Pré-image).
Funded by the European Union. Views and opinions expressed are however those of the authors only and do not necessarily reflect those of the European Union or European Commission. Neither the European Union nor the granting authority can be held responsible for them.
 This work received funding from the European Union’s Horizon Europe research and innovation program under grant agreement 101120237 (ELIAS).

\bibliography{references}
\bibliographystyle{apalike}

\section*{Checklist}

 \begin{enumerate}

 \item For all models and algorithms presented, check if you include:
 \begin{enumerate}
   \item A clear description of the mathematical setting, assumptions, algorithm, and/or model. [Yes] See \cref{sec:IOKR,sec:SISOKR,sec:theory,apx:not_def,sec-apx:vvRKHS,apx:preliminary_res}.
   \item An analysis of the properties and complexity (time, space, sample size) of any algorithm. [Yes] See \cref{sec:IOKR,sec:SISOKR,sec:expes,table:complexities}.
   \item (Optional) Anonymized source code, with specification of all dependencies, including external libraries. [Yes]
 \end{enumerate}

 \item For any theoretical claim, check if you include:
 \begin{enumerate}
   \item Statements of the full set of assumptions of all theoretical results. [Yes] See \cref{sec:theory}.
   \item Complete proofs of all theoretical results. [Yes] See \cref{apx:preliminary_res,apx:th1,apx:th2,apx:prob_bound,apx:aux_res}.
   \item Clear explanations of any assumptions. [Yes] See \cref{sec:theory}.     
 \end{enumerate}

 \item For all figures and tables that present empirical results, check if you include:
 \begin{enumerate}
   \item The code, data, and instructions needed to reproduce the main experimental results (either in the supplemental material or as a URL). [Yes] See \cref{sec:expes}.
   \item All the training details (e.g., data splits, hyperparameters, how they were chosen). [Yes] See \cref{sec:expes,apx:multi-label}.
   \item A clear definition of the specific measure or statistics and error bars (e.g., with respect to the random seed after running experiments multiple times). [Yes] See \cref{sec:expes}.
   \item A description of the computing infrastructure used. (e.g., type of GPUs, internal cluster, or cloud provider). [No]
 \end{enumerate}

 \item If you are using existing assets (e.g., code, data, models) or curating/releasing new assets, check if you include:
 \begin{enumerate}
   \item Citations of the creator If your work uses existing assets. [Yes] See \cref{sec:expes}.
   \item The license information of the assets, if applicable. [Not Applicable]
   \item New assets either in the supplemental material or as a URL, if applicable. [Not Applicable]
   \item Information about consent from data providers/curators. [Not Applicable]
   \item Discussion of sensible content if applicable, e.g., personally identifiable information or offensive content. [Not Applicable]
 \end{enumerate}

 \item If you used crowdsourcing or conducted research with human subjects, check if you include:
 \begin{enumerate}
   \item The full text of instructions given to participants and screenshots. [Not Applicable]
   \item Descriptions of potential participant risks, with links to Institutional Review Board (IRB) approvals if applicable. [Not Applicable]
   \item The estimated hourly wage paid to participants and the total amount spent on participant compensation. [Not Applicable]
 \end{enumerate}

 \end{enumerate}

\clearpage
\onecolumn
\appendix

\section{NOTATIONS AND DEFINITIONS}\label{apx:not_def}

In this section, we remind some important notations and definitions.

\paragraph{Setting.} In the following, we consider $\bmX$ and $\bmY$ to be Polish spaces. We denote by $\rho$ the unknown data distribution on $\bmX \times \bmY$. We denote by $\rhox$ and $\rhoy$ the marginal distributions of the inputs and outputs, respectively.

\paragraph{Linear algebra notation. } For an operator $A$, $A^{\#}$ is its adjoint, $\sigma_{\max}(A)$ its largest eigenvalue, and $\sigma_k(A)$ its $k^{\text{th}}$ largest eigenvalue (if $A$ admits an eigendecomposition). Let $\bmB\left(E\right)$ be the space of bounded linear operators in a separable Hilbert space $E$, given positive semi-definite operators $A, B \in \bmB\left(E\right)$, $A \preceq B$ if $B - A$ is positive semidefinite. For any $t > 0$ and $A : E \rightarrow E$, $A_t = A + t I_E$. Let $M$ be a matrix, $M_{i:}$ denotes its $i^{\text{th}}$ row and $M_{:j}$ its $j^{\text{th}}$ column, and $M^\dagger$ denotes its Moore-Penrose inverse.

\paragraph{Notation for simplified bounds.} To keep the dependencies of a bound only in the parameters of interest, for $a, b \in \reals $ we note $a \lesssim b$ as soon as there exists a constant $c >0$ independents of the parameters of interest such that $a \leq c \times b$.

\paragraph{Least-squares notation. }

For any function $h : \bmX \rightarrow \Hy$, its least-squares expected risk is given by
\begin{equation}
    \bmE\left(h\right) = \E_\rho\left[\left\|h\left(x\right)-\psiy\left(y\right)\right\|_{\Hy}^2\right]\,.
\end{equation}
The measurable minimizer of $\bmE$ is given by $h^*\left(x\right)=\E_{\rho\left(y|x\right)}\left[\psiy\left(y\right)\right]$ \citep[Lemma A.2]{ciliberto2020general}.

\paragraph{RKHS notation.} We denote by $\Hx$ and $\Hy$ the RKHSs associated to the input $\kernelx:\bmX \times \bmX \to \reals$ and output $\kernely: \bmY \times \bmY \to \reals$ kernels, respectively. We denote by $\psix: \bmX \to \Hx$ and $\psiy: \bmY \to \Hy$ the canonical feature maps $\psix(x) = \kernelx(x,.)$ and $\psiy(y) =  \kernely(y,.)$, respectively. We denote by $\bmH$ the vv-RKHS associated to the operator-valued kernel $\bmK = k I_{\Hy}$. We denote $\hat{h} \in \bmH$ the KRR estimator trained with $n$ couples $(x_i, y_i)_{i=1}^n$ i.i.d. from $\rho$.

\vspace{1em}
\paragraph{Kernel ridge operators. } We define the following operators.
\begin{itemize}
    \item $S : f \in \Hx \mapsto \langle f, \psix(\cdot)\rangle_{\Hx} \in L^2\left(\bmX, \rhox\right)$
    \item $T : f \in \Hy \mapsto \langle f, h^*(\cdot)\rangle_{\Hy} \in L^2\left(\bmX, \rhox\right)$
    \item $\covx = \E_x\left[\psix(x) \otimes \psix(x)\right]$ and $\covy = \E_y\left[\psiy(y) \otimes \psiy(y)\right]$,
    \item $\SX : f \in \Hx \mapsto \frac{1}{\sqrt{n}} \left(f\left(x_1\right), \ldots, f\left(x_n\right)\right)^\top \in \reals^n$,
    \item $\SX^{\#} : \alpha \in \reals^n \mapsto \frac{1}{\sqrt{n}} \sum_{i=1}^n \alpha_i \psix(x_i) \in \Hx$,
    \item $\SY : f \in \Hy \mapsto \frac{1}{\sqrt{n}} \left(f\left(y_1\right), \ldots, f\left(y_n\right)\right)^\top \in \reals^n$,
    \item $\SY^{\#} : \alpha \in \reals^n \mapsto \frac{1}{\sqrt{n}} \sum_{i=1}^n \alpha_i \psiy(y_i) \in \Hx$,
\end{itemize}

\paragraph{Sketching operators. }

\begin{itemize}
\item We denote $\sketchx \in \reals^{\mx \times n}$ and $\sketchy \in \reals^{\my \times n}$ the input and output sketch matrices with $\mx < n$ and $\my < n$,
\item $\sketchempcovx = \SX^{\#} \sketchx^\top \sketchx \SX$ and $\sketchempcovy = \SY^{\#} \sketchy^\top \sketchy \SY$,
\item $\sketchGramX = \sketchx \GramX \sketchx^\top$ and $\sketchGramY = \sketchy \GramY \sketchy^\top$.
\end{itemize}

\section{REMINDERS ABOUT VECTOR-VALUED REPRODUCING KERNEL HILBERT SPACES AND OPERATOR-VALUED KERNELS}\label{sec-apx:vvRKHS}

%

We recall the definitions of an OVK and its vv-RKHS.
Let $\bmF$ be a Hilbert space and $\bmL(\bmF)$ the set of bounded linear operators on $\bmF$.
\begin{definition}[Operator-valued kernel]
    An OVK is a mapping $\bmK: \bmX \times \bmX \rightarrow \bmL(\bmF)$ such that
    \begin{itemize}
        \item $\bmK\left(x, x^{\prime}\right)=\bmK\left(x^{\prime}, x\right)^\#$ for all $(x, x') \in \bmX^2$;
        \item $\sum_{i, j=1}^{n} \left\langle \varphi_{i}, \bmK\left(x_{i}, x_{j}\right) \varphi_{j}) \right\rangle_{\bmF} \geqslant 0$ for all $n \in \mathbb{N}$ and $\left(x_{i}, \varphi_{i})\right)_{i=1}^{n} \in(\bmX \times \bmF)^{n}$.
    \end{itemize}
\end{definition}
Similarly to the scalar case, an OVK is uniquely associated to a vv-RKHS $\bmH$.
\begin{theorem}[vector-valued RKHS]
    Let $\bmK$ be an OVK. There is a unique Hilbert space $\bmH$ of functions from $\bmX$ to $\bmF$, the vv-RKHS of $\bmK$, such that for all $x\in\bmX$, $\varphi \in \bmF$ and $f \in \bmH$
    \begin{itemize}
        \item $x^\prime \mapsto \bmK\left(x, x^\prime\right) \varphi \in \bmF$;
        \item $\langle f, \bmK\left(\cdot, x\right) \varphi \rangle_{\bmH} = \left\langle f(x), \varphi \right\rangle_{\bmF}$ (reproducing property).
    \end{itemize}
\end{theorem}

\section{PRELIMINARY RESULTS}
\label{apx:preliminary_res}

In this section, we present useful preliminary results about kernel ridge operators and sketching properties, as well as the proof \cref{prop:SISOKR_exp} that give the expressions of the SISOKR estimator.

\paragraph{Useful kernel ridge operators properties. } The following results hold true.
\begin{itemize}
    \item $\empcovx = \frac{1}{n} \sum_{i=1}^n \psix(x_i) \otimes \psix(x_i) = \SX^{\#} \SX$ and $\empcovy = \frac{1}{n} \sum_{i=1}^n \psiy(y_i) \otimes \psiy(y_i) = \SY^{\#} \SY$,
    \item $\GramX = n \SX \SX^{\#}$ and $\GramY = n \SY \SY^{\#}$,
    \item Under the attainability assumption \cite[Lemma B.2, B.4, B.9]{ciliberto2020general} show that:
    \begin{itemize}[label={\tiny $\blacksquare$}]
        \item For all $x \in \bmX$, $\hat{h}(x) = \widehat{H} \psix(x)$, where $\widehat{H} = \SY^{\#} \SX \empcovxlambda^{-1}$.
        \item $\E[\|\hat h(x) - h^*(x)\|^{2}]^{1/2} = \|(\widehat H - H)S^\#\|_{\HS}$.

    \end{itemize}

\end{itemize}

\paragraph{Useful sketching properties. } We remind some useful notations and provide the expression of $\projZ$, leading to the expression of the SISOKR estimator.
%



\paragraph{Expression of $\projZ$.} Let $\left\{\big(\sigma_i(\sketchGramZ), \tilde{\mathbf{v}}_{i}^Z\big), i \in[\mz]\right\}$ be the eigenpairs of $\sketchGramZ$ ranked in descending order of eigenvalues, $\ranksketchGramZ = \operatorname{rank}\big(\sketchGramZ\big)$, and for all $1 \leq i \leq \ranksketchGramZ$, $\tilde{e}_i^Z = \sqrt{\frac{n}{\sigma_i(\sketchGramZ)}} \SZ^{\#} \sketchz^\top \tilde{\mathbf{v}}_{i}^Z$.
\begin{proposition}\label{prop:proj_exp}
    The $\tilde{e}_i^Z$s are the eigenfunctions, associated to the eigenvalues $\sigma_i(\sketchGramZ)/n$ of $\sketchempcovz$. Furthermore, let $\widetilde{\bmH}_\bmZ = \spn\left(\tilde{e}_1^z, \ldots, \tilde{e}_{\ranksketchGramZ}^z\right)$, the orthogonal projector $\projZ$ onto $\widetilde{\bmH}_\bmZ$ writes as
    \begin{equation}\label{eq:proj_exp}
        \projZ = (\sketchz \SZ)^{\#} \left(\sketchz \SZ (\sketchz \SZ)^{\#}\right)^{\dagger} \sketchz \SZ\,.
    \end{equation}
\end{proposition}
\begin{proof}
    For $1 \leq i \leq \ranksketchGramZ$
    \begin{align}
        \sketchempcovz ~ \tilde{e}_i^Z &= \SZ^{\#} \sketchz^\top \sketchz \SZ \left(\sqrt{\frac{n}{\sigma_i(\sketchGramZ)}} \SZ^{\#} \sketchz^\top \tilde{\mathbf{v}}_{i}^Z\right) \\
        &= \sqrt{\frac{n}{\sigma_i(\sketchGramZ)}} \SZ^{\#} \sketchz^\top \left(\frac{1}{n} \sketchGramZ\right) \tilde{\mathbf{v}}_{i}^Z \\
        &= \frac{1}{\sqrt{n \sigma_i(\sketchGramZ)}} \SZ^{\#} \sketchz^\top \sigma_i(\sketchGramZ)\tilde{\mathbf{v}}_{i}^Z \\
        &= \frac{\sigma_i(\sketchGramZ)}{n} \tilde{e}_i^Z\,.
    \end{align}
    Moreover, we verify that $\spn\left(\tilde{e}_1^Z, \ldots, \tilde{e}_{\ranksketchGramZ}^Z\right)$ forms an orthonormal basis. Let $1 \leq i, j \leq \ranksketchGramZ$,
    \begin{align}
        \left\langle \tilde{e}_i^Z, \tilde{e}_j^Z \right\rangle_{\Hx} &= \left\langle \sqrt{\frac{n}{\sigma_i(\sketchGramZ)}} \SZ^{\#} \sketchz^\top \tilde{\mathbf{v}}_{i}^Z, \sqrt{\frac{n}{\sigma_j(\sketchGramZ)}} \SZ^{\#} \sketchz^\top \tilde{\mathbf{v}}_{j}^Z \right\rangle_{\Hz} \\
        &= \frac{n}{\sqrt{\sigma_i(\sketchGramZ) \sigma_j(\sketchGramZ)}} \tilde{\mathbf{v}}_{i}^{Z^\top} \sketchz \SZ \SZ^{\#} \sketchz^\top \tilde{\mathbf{v}}_{j}^Z \\
        &= \frac{n}{\sqrt{\sigma_i(\sketchGramZ) \sigma_j(\sketchGramZ)}} \tilde{\mathbf{v}}_{i}^{Z^\top} \left(\frac{1}{n} \sketchGramZ\right) \tilde{\mathbf{v}}_{j}^Z \\
        &= \frac{\sigma_j(\sketchGramZ)}{\sqrt{\sigma_i(\sketchGramZ) \sigma_j(\sketchGramZ)}} \tilde{\mathbf{v}}_{i}^{Z^\top} \tilde{\mathbf{v}}_{j}^Z \\
        &= \delta_{i j}\,,
    \end{align}
    where $\delta_{i j} = 0$ if $i \ne j$, and $1$ otherwise. \\ \\
    Finally, it is easy to check that the orthogonal projector onto $\spn\left(\tilde{e}_1^Z, \ldots, \tilde{e}_{\ranksketchGramZ}^Z\right)$, i.e. $\projZ : f \in \Hz \mapsto \sum_{i=1}^{\ranksketchGramZ} \left\langle f, \tilde{e}_i^Z \right\rangle_{\Hz} \tilde{e}_i^Z$ rewrites as
    \begin{equation}
        \projZ = n \SZ^{\#} \sketchz^\top \sketchGramZ^\dagger \sketchz \SZ = (\sketchz \SZ)^{\#} \left(\sketchz \SZ (\sketchz \SZ)^{\#}\right)^{\dagger} \sketchz \SZ\,.
    \end{equation}
\end{proof}
\begin{remark}
    With $\sketchx$ a sub-sampling matrix, we recover the linear operator $L_m$ introduced in \citet{Yang_NIPS2012_621bf66d} for the study of Nystr{\"o}m approximation and its eigendecomposition. Moreover, we also recover the projection operator $P_m$ from \citet{rudi2015less} and follow the footsteps of the proposed extension ``Nystr\"{o}m with sketching matrices''.
\end{remark}

\paragraph{Algorithm.} We here give the proof of \cref{prop:SISOKR_exp} that provides an expression of the SISOKR estimator $\tilde{h}$ as a linear combination of the $\psiy(y_i)$s.
\propsisokr*
\begin{proof}
    Recall that $\tilde{h}(x) = \projY \SY^{\#} \SX \projX(\projX \SX^\# \SX \projX + \lambda I_{\Hx})^{-1} \psix(x)$. By \cref{lemma:SIOKR_exp} and especially \eqref{eq:tildeH_algo}, we obtain that
    \begin{equation}
        \tilde{h}(x) = \sqrt{n} \projY \SY^{\#} \GramX \sketchx^\top \left(\sketchx \GramX^2 \sketchx^\top + n \lambda \sketchx \GramX \sketchx^\top\right)^\dagger \sketchx \SX \psix(x)\,.
    \end{equation}
    Finally, by \cref{lemma:ISOKR_exp} and with $\alpha(x) = \GramX \sketchx^\top \left(\sketchx \GramX^2 \sketchx^\top + n \lambda \sketchx \GramX \sketchx^\top\right)^\dagger \sketchx \SX \psix(x)$, we have that $\tilde{h}\left(x\right)=\sum_{i=1}^n \tilde{\alpha}_i\left(x\right) \psiy\left(y_i\right)$ where
    \begin{equation}
        \tilde{\alpha}\left(x\right) = \sketchy^\top \sketchGramY^\dagger \sketchy \GramY \GramX \sketchx^\top (\sketchx \GramX^2 \sketchx^\top + n \lambda \sketchGramX)^{\dagger} \sketchx \kernelvectX\,.
    \end{equation}
\end{proof}
Before stating and proving \cref{lemma:SIOKR_exp,lemma:ISOKR_exp}, and similarly to \citet{rudi2015less}, let $\sketchx \SX = U \Sigma V^{\#}$ be the SVD of $\sketchx \SX$ where $U : \reals^{\ranksketchGramX} \rightarrow \reals^{\mx}$, $\Sigma : \reals^{\ranksketchGramX} \rightarrow \reals^{\ranksketchGramX}$, $V : \reals^{\ranksketchGramX} \rightarrow \Hx$, and $\Sigma = \diag(\sigma_1(\sketchx \SX), \ldots, \sigma_{\ranksketchGramX}(\sketchx \SX))$ with $\sigma_1(\sketchx \SX) \geq \ldots \geq \sigma_{\ranksketchGramX}(\sketchx \SX) > 0$, $U U^\top = I_{\ranksketchGramX}$ and $V^{\#} V = I_{\ranksketchGramX}$.
We are now ready to prove the following lemma for the expansion induced by input sketching.
\begin{lemma}\label{lemma:SIOKR_exp}
     Let $\widetilde{H} = \projY \SY^{\#} \SX \projX(\projX \SX^\# \SX \projX + \lambda I_{\Hx})^{-1}$. The following two expansions hold true
    \begin{equation}\label{eq:tildeH_theory}
        \widetilde{H} = \projY \SY^{\#} \SX \tilde{\eta}(\empcovx)\,,
    \end{equation}
    where $\tilde{\eta}(\empcovx) = V (V^\# \empcovx V + \lambda I_{\Hx})^{-1}V^\#$ and for algorithmic purposes
    \begin{equation}\label{eq:tildeH_algo}
        \widetilde{H} = \sqrt{n} \projY \SY^{\#} \GramX \sketchx^\top \left(\sketchx \GramX^2 \sketchx^\top + n \lambda \sketchx \GramX \sketchx^\top\right)^\dagger \sketchx \SX\,.
    \end{equation}
\end{lemma}
\begin{proof} 
    Let us prove \eqref{eq:tildeH_theory} first.
    \begin{align}
        \tilde H &= \projY \SY^{\#}\SX \projX (\projX \SX^\# \SX \projX + \lambda I_{\Hx})^{-1} \\
        &= \projY \SY^{\#} \SX VV^\#(VV^\# \SX^\# \SX VV^\# + \lambda I_{\Hx})^{-1} \\
        &= \projY \SY^{\#} \SX V (V^\# \hat \covx V + \lambda I_{\Hx})^{-1}V^\# \\
        &= \projY \SY^{\#} \SX \tilde{\eta}(\empcovx)\,,
    \end{align}
    using the so-called push-through identity $(I + UV)^{-1}U = U(I + VU)^{-1}$. \\ \\
    Now, we focus on proving \eqref{eq:tildeH_algo}. First, we have that
    \begin{equation}
        \tilde H = \projY \SY^{\#} \SX V (V^\# \empcovxlambda V)^{\dagger}V^\#\,.
    \end{equation}
    Then, using the fact that $U$ has orthonormal columns, $U^\top$ has orthonormal rows and $\Sigma$ is a full-rank matrix, together with the fact that $U U^\top = I_{\ranksketchGramX}$ and $V^{\#} V = I_{\ranksketchGramX}$, we have that,
    \begin{equation}
        \tilde H = \projY \SY^{\#} \SX V \Sigma U^\top \left(U \Sigma V^\# \empcovxlambda V \Sigma U^\top\right)^\dagger U \Sigma V^\#\,.
    \end{equation}
    Then, since $\sketchx \SX = U \Sigma V^{\#}$,
    \begin{equation}
        \tilde H = \projY \SY^{\#} \SX (\sketchx \SX)^\# \left(\sketchx \SX \left(\empcovx + \lambda I_{\Hx}\right) (\sketchx \SX)^\#\right)^\dagger \sketchx \SX\,.
    \end{equation}
    Finally, using the fact that $\empcovx = \SX^\# \SX$ and $\GramX = n \SX \SX^\#$, we obtain that
    \begin{equation}
        \widetilde{H} = \sqrt{n} \projY \SY^\# \GramX \sketchx^\top \left(\sketchx \GramX^2 \sketchx^\top + n \lambda \sketchx \GramX \sketchx^\top\right)^\dagger \sketchx \SX\,.
    \end{equation}
\end{proof}
Now we state and prove the lemma for the expansion induced by output sketching.
\begin{lemma}\label{lemma:ISOKR_exp}
    For all $x \in \bmX$, for any $h \in \bmH$ that writes as $h(x) = \sqrt{n} \projY \SY^\# \alpha(x)$ with $\alpha : \bmX \rightarrow \reals^n$, then $h(x) = \sum_{i=1}^n \sketchy^\top \sketchGramY^\dagger \sketchy \GramY \alpha(x) \psiy(y_i)$.
\end{lemma}
\begin{proof}
    \begin{align}
        h(x) &= \sqrt{n} \projY \SY^\# \alpha(x) \\
        &= \sqrt{n} \SY^{\#} \sketchy^\top \sketchGramY^{\dagger} \sketchy \left(n \SY \SY^{\#}\right) \alpha(x) \\
        &= \sqrt{n} \SY^{\#} \sketchy^\top \sketchGramY^{\dagger} \sketchy \GramY \alpha(x) \\
        &= \sum_{i=1}^n \sketchy^\top \sketchGramY^\dagger \sketchy \GramY \alpha(x) \psiy(y_i)\,.
    \end{align}
\end{proof}
\section{SISOKR EXCESS-RISK BOUND}\label{apx:th1}

In this section, we provide the proof of \cref{th:sketched_ridge} which gives a bound on the excess-risk of the proposed approximated regression estimator with both input and output sketching (SISOKR).


\thmsisokr*

\begin{proof}  Our proofs consist of decompositions and then applying the probabilistic bounds given in Section \ref{apx:prob_bound}. 

We have
\begin{align}
    \E[\|\tilde h(x) - h^*(x)\|^{2}]^{1/2} = \|(\tilde H - H)S^\#\|_{\HS}
\end{align}
with $\tilde H = \widetilde P_Y \SY^{\#}\SX \tilde{\eta}(\empcovx)$.

\vspace{1em}
Then, defining $H_{\lambda}= H\covx(\covx + \lambda I)^{-1}$, we decompose
\begin{equation}\label{eq:decompo_SISOKR}
    \tilde H - H = \widetilde P_Y \left(\SY^\# \SX - H_{\lambda} \empcovx\right) \tilde{\eta}(\empcovx) + \widetilde P_Y H_{\lambda}\left(\empcovx \tilde{\eta}(\empcovx) - I_{\Hx}\right) + \left(\widetilde P_Y H_\lambda - H\right)
\end{equation}
such that
\begin{equation}
   \|(\tilde H - H)S^\#\|_{\HS} \leq (A) + (B) + (C)
\end{equation}
with
\begin{align}
   &(A) = \left\|\left(\SY^\# \SX - H_\lambda \empcovx\right) \tilde{\eta}(\empcovx) \covx^{1/2}\right\|_{\HS} \label{eq:reg_error} \\
   &(B) = \left\|H_\lambda\left(\empcovx \tilde{\eta}(\empcovx) - I_{\Hx}\right) \covx^{1/2}\right\|_{\HS} \label{eq:in_sketch_error} \\
   &(C) = \left\|(\widetilde P_Y H_\lambda - H)\covx^{1/2}\right\|_{\HS} \label{eq:out_sketch_error}
\end{align}

Then, from \cref{lem:bound_A,lem:bound_B,lem:bound_C}, we obtain 
\begin{align}
    \|(\tilde H - H)S^\#\|_{\HS} \leq  2\sqrt{3} M\log(4/\delta) n^{-\frac{1}{2(1 + \gammax)}} &+ 2 \sqrt{3} \|H\|_{\HS} \|(I - \widetilde P_X) \covx^{1/2} \|_{\op}\\
    &+ \E_y\left[\left\|\left(\projY -I_{\Hy}\right)\psiy(y)\right\|_{\Hy}^2\right]^{1/2}.
\end{align}
Then, notice that
\begin{align}
    \big\|(I - \widetilde P_X) \covx^{1/2} \big\|_{\op} &\leq \big\|(I - \widetilde P_X) \covx^{1/2} \big\|_{\HS}\\
    &= \E_x\left[\left\|\left(\projX -I_{\Hx}\right)\psix(x)\right\|_{\Hx}^2\right]^{1/2}.
\end{align}
We conclude by defining
\begin{align}
    c_1 &= 2\sqrt{3}M,\\
    c_2 &= 2\sqrt{3}\|H\|_{\HS}.
\end{align}

\end{proof}

\begin{lemma}[Bound (A)]\label{lem:bound_A} Let $\delta \in [0,1]$, $n \in \mathbb{N}$ sufficiently large such that $\lambda = n^{-1/(1+\gamma)} \geq \frac{9\kappax^2}{n} \log(\frac{n}{x})$ Under our set of assumptions, the following holds with probability at least $1-\delta$
\begin{align}
    (A) \leq 2 M \log(4/\delta) n^{-\frac{1}{2(1 + \gammax)}}.
\end{align}
where the constant $M$ depends on $\kappay, \|H\|_{\HS}, \delta$.

\end{lemma}

\begin{proof}

We have 

\begin{align}
    (A) &\leq \underbrace{\left\|\left(\SY^\# \SX - H_\lambda \empcovx\right) \covxlambda^{-1/2}\right\|_{\HS}}_{(A.1)} \times \underbrace{\|\covxlambda^{1/2} \tilde{\eta}(\empcovx) \covx^{1/2}\|_{\op}}_{(A.2)}
\end{align}






Moreover, we have
\begin{align}
    (A.2) &\leq  \|\empcovxlambda^{1/2} \tilde{\eta}(\empcovx) \empcovxlambda^{1/2} \|_{\op} \|\empcovxlambda^{-1/2} \covxlambda^{1/2}\|_{\op}^2 \|\covxlambda^{-1/2} \covx^{1/2}\|_{\op}\\
    &\leq \|\empcovxlambda^{1/2} \tilde{\eta}(\empcovx) \empcovxlambda^{1/2} \|_{\op} \|\empcovxlambda^{-1/2} \covxlambda^{1/2}\|_{\op}^2 
\end{align}
because $\|\covxlambda^{-1/2} \covx^{1/2}\|_{\op} \leq 1$.


Finally, by using the probabilistic bounds given in \cref{lem:proba_bound_1,lem:proba_bound_2}, and Lemma \ref{lem:proba_bound_4}, we obtain
\begin{align}
    (A) \leq 2 M \log(4/\delta) n^{-\frac{1}{2(1 + \gammax)}}.
\end{align}

\end{proof}

\begin{lemma}[Bound (B)]\label{lem:bound_B} If $\frac{9}{n} \log \frac{n}{\delta} \leq \lambda \leq \|\covx\|_{\op}$, then with probability $1- \delta$
\begin{align}
    (B) \leq 2\sqrt{3}\|H\|_{\HS}(\lambda^{1/2} + \|(I - \widetilde P_X) \covx^{1/2} \|_{\op})
\end{align}
\end{lemma}

\begin{proof}

We do a similar decomposition than in \citet[Theorem 2]{rudi2015less}:
\begin{align}
    \empcovx \tilde{\eta}(\empcovx) - I_{\Hx} &= \empcovxlambda \tilde{\eta}(\empcovx) - \lambda \tilde{\eta}(\empcovx)  -  I_{\Hx}\\
    &= (I - \widetilde P_X) \empcovxlambda \tilde{\eta}(\empcovx) + \widetilde P_X \empcovxlambda \tilde{\eta}(\empcovx)  - \lambda \tilde{\eta}(\empcovx)  -  I_{\Hx}\\
    &= (I - \widetilde P_X) \empcovxlambda \tilde{\eta}(\empcovx)  - \lambda \tilde{\eta}(\empcovx)  -  (\widetilde P_X  - I_{\Hx})\,,
\end{align}
as $\widetilde P_X \empcovxlambda \tilde{\eta}(\empcovx) = \widetilde P_X $.

Then, we have
\begin{align}
    (B) &\leq \left\|H_\lambda\right\|_{\HS} \left\|\left(\empcovx \tilde{\eta}(\empcovx) - I_{\Hx}\right) \covx^{1/2}\right\|_{\op}\\
    &\leq \left\|H_\lambda\right\|_{\HS} \left(\big\|(I - \widetilde P_X) \empcovxlambda \tilde{\eta}(\empcovx)\covx^{1/2}\big\|_{\op}  + \lambda \big\|\tilde{\eta}(\empcovx)\covx^{1/2}\big\|_{\op} + \big\|(\widetilde P_X  - I_{\Hx})\covx^{1/2}\big\|_{\op}\right)
\end{align}

But,
\begin{align}
    \left\|H_\lambda\right\|_{\HS} &\leq \left\|H\left(\covx \covxlambda^{-1} - I_{\Hx}\right)\right\|_{\HS} +  \left\|H\right\|_{\HS}\\
    &= \left\|H\left(\covx - \covxlambda\right)\covxlambda^{-1}\right\|_{\HS} +  \left\|H\right\|_{\HS}\\
    &= \lambda \left\|H \covxlambda^{-1}\right\|_{\HS} +  \left\|H\right\|_{\HS}\\
    &\leq 2 \|H\|_{\HS}.
\end{align}

And,
\begin{align}
    \big\|(I - \widetilde P_X) \empcovxlambda \tilde{\eta}(\empcovx)\covx^{1/2}\big\|_{\op} &\leq  \big\|(I - \widetilde P_X) \empcovxlambda^{1/2} \big\|_{\op} \big\|\empcovxlambda^{1/2} \tilde{\eta}(\empcovx) \empcovxlambda^{1/2}\big\|_{\op} \big\|\empcovxlambda^{-1/2} \covx^{1/2}\big\|_{\op}.
\end{align}

And,
\begin{align}
    \big\|(I - \widetilde P_X) \empcovxlambda^{1/2} \big\|_{\op} &\leq  \big\|(I - \widetilde P_X) {C}_{X\lambda}^{1/2} \big\|_{\op} \big\| \covxlambda^{-1/2}  \empcovxlambda^{1/2} \big\|_{\op}.
\end{align}

And,
\begin{align}
    \big\|(I - \widetilde P_X) \covxlambda^{1/2} \big\|_{\op} &\leq \big\|(I - \widetilde P_X) \covx^{1/2} \big\|_{\op} + \lambda^{1/2}.
\end{align}

Moreover,
\begin{align*}
    \left\|\lambda \tilde{\eta}\left(\empcovx\right) \covx^{1/2}\right\|_{\op} &\leq \lambda \left\|\empcovxlambda^{-1/2}\right\|_{\op} \left\|\empcovxlambda^{1/2} \tilde{\eta}\left(\empcovx\right) \empcovxlambda^{1/2}\right\|_{\op} \left\|\empcovxlambda^{-1/2} \covxlambda^{1/2}\right\|_{\op} \left\|\covxlambda^{-1/2} \covx^{1/2}\right\|_{\op}\\
    &\leq \lambda^{1/2} \left\|\empcovxlambda^{1/2} \tilde{\eta}\left(\empcovx\right) \empcovxlambda^{1/2}\right\|_{\op} \left\|\empcovxlambda^{-1/2} \covxlambda^{1/2}\right\|_{\op}.
\end{align*}

\paragraph{Conclusion. } Using the probabilistic bounds given in Lemmas \ref{lem:proba_bound_2}, \ref{lem:proba_bound_3}, and Lemma \ref{lem:proba_bound_4}, we obtain 
\begin{align}
    (B) \leq 4\sqrt{3}\|H\|_{\HS}\big(\lambda^{1/2} + \big\|(I - \widetilde P_X) \covx^{1/2} \big\|_{\op}\big) 
\end{align}

\end{proof}

\begin{lemma}[Bound (C)]\label{lem:bound_C} We have
\begin{align}
    (C) &\leq \E_y\left[\left\|\left(\projY -I_{\Hy}\right)\psiy(y)\right\|_{\Hy}^2\right]^{1/2} + \lambda^{1/2} \left\|H\right\|_{\HS}.
\end{align}
\end{lemma}

\begin{proof} We have
\begin{align}
    (C) &= \left\|(\widetilde P_Y H(I_{\Hx} - \lambda \covxlambda^{-1}) - H) \covx^{1/2}\right\|_{\HS}\\
    &\leq \left\|(\widetilde P_Y -I_{\Hy}) H \covx^{1/2}\right\|_{\HS} + \lambda^{1/2} \left\|H\right\|_{\HS}\\
    &= \E[\|(\widetilde P_Y -I_{\Hy})h^*(x)\|_{\Hy}^2]^{1/2} + \lambda^{1/2} \left\|H\right\|_{\HS}.
\end{align}
We conclude the proof as follows. Using the fact that $h^*\left(x\right)=\E_{\rho\left(y|x\right)}\left[\psiy\left(y\right)\right]$, the linearity of $\projY -I_{\Hy}$ and the convexity of $\left\|\cdot\right\|_{\Hy}^2$, by the Jensen's inequality we obtain that
    \begin{align}
        \E_x\left[\left\|\left(\projY -I_{\Hy}\right)h^*\left(x\right)\right\|_{\Hy}^2\right] &= \E_x\left[\left\|\left(\projY -I_{\Hy}\right)\E_{\rho\left(y|x\right)}\left[\psiy\left(y\right)\right]\right\|_{\Hy}^2\right] \\
        &= \E_x\left[\left\|\E_{\rho\left(y|x\right)}\left[\left(\projY -I_{\Hy}\right)\psiy\left(y\right)\right]\right\|_{\Hy}^2\right] \\
        &\leq \E_x\left[\E_{\rho\left(y|x\right)}\left[\left\|\left(\projY -I_{\Hy}\right)\psiy\left(y\right)\right\|_{\Hy}^2\right]\right] \\
        &= \E_y\left[\left\|\left(\projY -I_{\Hy}\right)\psiy(y)\right\|_{\Hy}^2\right]\,.
    \end{align}

\end{proof}
\section{SKETCHING RECONSTRUCTION ERROR}\label{apx:th2}

We provide here a bound on the reconstruction error of a sketching approximation.


\thmsketching*

\begin{proof}
    For $t>0$, we have
    \begin{align}
        \E_z\left[\left\|\left(\projZ -I_{\Hz}\right)\psiz(z)\right\|_{\Hz}^2\right]
        &= \Tr\left(\left(\projZ -I_{\Hz}\right) \E_z\left[\psiz(z) \otimes \psiz(z)\right]\right) \\
        &= \left\|\left(\projZ -I_{\Hz}\right) \covz^{1/2}\right\|_{\HS}^2 \\
        &\leq \left\|\left(\projZ -I_{\Hz}\right)\empcovzt^{1/2}\right\|_{\op}^2 \left\|\empcovzt^{-1/2} \covzt^{1/2}\right\|_{\op}^2 \left\|\covzt^{-1/2} \covz^{1/2}\right\|_{\HS}^2\,.
    \end{align}
    Lemma \ref{lem:proba_bound_2} gives that, for $\delta \in \left(0, 1\right)$, if $\frac{9}{n}\log\left(\frac{n}{\delta}\right) \leq t \leq \left\|\covz\right\|_{\op}$, then with probability $1-\delta$
    \begin{equation}\label{eq:out_emp_exp_cov_op}
        \left\|\empcovzt^{-1/2} \covzt^{1/2}\right\|_{\op}^2 \leq 2\,.
    \end{equation}
    Moreover, since $\left\|{\covzt}^{-1/2} \covz^{1/2}\right\|_{\HS}^2 = \Tr\left({\covzt}^{-1}{\covz}\right) = \deffz(t)$, \cref{lem:int_dim} gives that
    \begin{equation}
        \left\|{\covzt}^{-1/2} \covz^{1/2}\right\|_{\HS}^2 \leq \Qz t^{-\gammaz}\,.
    \end{equation}
    Then, using the Lemma \ref{lemma:beta}, and multiplying the bounds, gives
    \begin{align}
        \E_y\left[\left\|\left(\projZ -I_{\Hy}\right)\psiz(z)\right\|_{\Hz}^2\right] &\leq 6 \Qz t^{1-\gammaz}.
    \end{align}
    Finally, choosing $t = n^{-\frac{1}{1+\gammaz}}$, defining $c_3=6 \Qz$, $c_4=576 {\mathfrak{C}}^2 \bz \Qz$, and noticing $\bmN_{\bmZ}^{\infty}(t) \leq \bz\Qz t^{-(\gammaz + \muz)}$ (from \cref{lem:int_dim,lem:emb}), allows to conclude the proof.
    
\end{proof}

\begin{lemma}\label{lemma:beta} Let $\bmN_{\bmZ}^{\infty}(t)$ be as in \cref{def:dim_eff}. For all $\delta \in \left(0, 1/e\right]$, $\frac{9}{n}\log\left(\frac{n}{\delta}\right) \leq t \leq \left\|\covz\right\|_{\op} - \frac{9}{n}\log\left(\frac{n}{\delta}\right)$ and $\mz \geq  \max\left(432 {\mathfrak{C}}^2 \nuz^2 \bmN_{\bmZ}^{\infty}(t), 576 {\mathfrak{C}}^2 \nuz^4 \log\left(1/\delta\right)\right)$, with probability at least $1-\delta$, 
    \begin{equation}
        \left\|\left(\projZ -I_{\Hz}\right)\empcovzt^{1/2}\right\|_{\op}^2 \leq 3t\,.
    \end{equation}
\end{lemma}
\begin{proof}
    Using Propositions 3 and 7 from \citet{rudi2015less}, we have, for $t>0$,
    \begin{equation}\label{eq:diff_op_norms}
        \left\|\left(\projZ -I_{\Hz}\right)\empcovzt^{1/2}\right\|_{\op}^2 \leq \frac{t}{1-\betaz(t)}\,,
    \end{equation}
    with $\betaz(t)=\sigma_{\max}\left(\empcovzt^{-1/2}\left(\empcovz-\sketchempcovz\right)\empcovzt^{-1/2}\right)$.

    Now, applying \cref{lemma:bound_beta_sub-Gaussian}, with the condition 
    \begin{align}
        \mz \geq \max\left(432 {\mathfrak{C}}^2 \nuz^2 \bmN_{\bmZ}^{\infty}(t), 576 {\mathfrak{C}}^2 \nuz^4 \log\left(1/\delta\right)\right),
    \end{align}
    we obtain $\betaz(t) \leq 2/3$, which gives
    \begin{align}
        \left\|\left(\projZ -I_{\Hy}\right)\empcovzt^{1/2}\right\|_{\op}^2 \leq 3t\,.
    \end{align}
    
\end{proof}

\begin{lemma}\label{lemma:bound_beta_sub-Gaussian}
    Let $\sketchz$ be as in \cref{def:subG_sketch} and $\bmN_{\bmZ}^{\infty}(t)$ as in \cref{def:dim_eff}. For all $\delta \in \left(0, 1/e\right]$, $\frac{9}{n}\log\left(\frac{n}{\delta}\right) \leq t \leq \left\|\covz\right\|_{\op} - \frac{9}{n}\log\left(\frac{n}{\delta}\right)$ and $\mz \geq \max\left(6 \bmN_{\bmZ}^{\infty}(t), \log\left(1/\delta\right)\right)$, with probability at least $1-\delta$,
    \begin{equation}
        \left\|\empcovzt^{-1/2}\left(\empcovz-\sketchempcovz\right)\empcovzt^{-1/2}\right\|_{\op} \leq \mathfrak{C} \frac{2 \sqrt{2} \nuz \sqrt{6 \bmN_{\bmZ}^{\infty}(t)} + 8 \nuz^2 \sqrt{\log\left(1/\delta\right)}}{\sqrt{\mz}}\,,
    \end{equation}
    where $\mathfrak{C}$ is a universal constant independent of $\bmN_{\bmZ}^{\infty}(t)$, $\delta$ and $\mz$.
\end{lemma}
\begin{proof} We define the following random variables

\begin{align}
    W_i = \sqrt{\frac{\mz}{n}} \sum_{j=1}^n (\sketchz)_{i j} \empcovzt^{-1/2} \psiz(z_j) \in \Hz \quad \text{ for } i=1, \dots \mz.
\end{align}

In order to use the concentration bound given in Theorem \ref{th:kolt_subG}, we show that the $W_i$s are i.i.d. weakly square integrable centered random vectors with covariance operator $\Sigma$, sub-Gaussian, and pre-Gaussian. 
    
\paragraph{The $W_i$s are weakly square integrable. } Let $u \in \Hz$ and $v = \empcovzt^{-1/2} u$, we have that $\langle W_i, u \rangle_{\Hz} = \sqrt{\frac{\mz}{n}} \sum_{j=1}^n (\sketchz)_{i j} v(z_j)$. Hence, using the definition of a sub-Gaussian sketch, we have
\begin{align}
    \left\|\langle W_i, u \rangle_{\Hz}\right\|_{L_2(\Prob)}^2 &= \E_{\sketchz}\left[|\langle W_i, u \rangle_{\Hz}|^2\right]\\
    &= \frac{1}{n} \sum_{j=1}^n v(z_j)^2 \\
    &< + \infty\,.
\end{align}
%

\paragraph{The $W_i$s are sub-Gaussian. } Let $c \in \reals$, using the independence and sub-Gaussianity of the $R_{z_{ij}}$, we have
\begin{align}
    \E_{\sketchz}\left[\exp\left(c\langle W_i, u \rangle_{\Hz}\right)\right] &= \E_{\sketchz}\left[\exp\left(\sum_{j=1}^n c \sqrt{\frac{\mz}{n}} R_{z_{i j}} v(z_j)\right)\right] \\
    &= \prod_{j=1}^n \E_{\sketchz}\left[\exp\left(c \sqrt{\frac{\mz}{n}} R_{z_{i j}} v(z_j)\right)\right] \\
    &\leq \prod_{j=1}^n \exp\left(\frac{c^2 \mz v(z_j)^2}{2 n} \frac{\nuz^2}{\mz}\right) \\
    &= \exp\left(\frac{c^2 \nuz^2}{2 n} \sum_{j=1}^n v(z_j)^2\right) \\
    &= \exp\left(\frac{c^2 \nuz^2}{2} \left\|\langle W_i, u \rangle_{\Hz}\right\|_{L_2(\Prob)}^2\right)\,.
\end{align}
Hence, $\langle W_i, u \rangle_{\Hz}$ is a $\frac{1}{2} \nuz^2 \left\|\langle W_i, u \rangle_{\Hz}\right\|_{L_2(\Prob)}^2$-sub-Gaussian random variable. Then, the Orlicz condition of sub-Gaussian random variables gives
\begin{equation}
    \E_{\sketchz}\left[\exp\left(\frac{\langle W_i, u \rangle_{\Hz}^2}{8 \nuz^2 \left\|\langle W_i, u \rangle_{\Hz}\right\|_{L_2(\Prob)}^2}\right) - 1\right] \leq 1\,.
\end{equation}
We deduce that
\begin{equation}
    \left\|\langle W_i, u \rangle_{\Hz}\right\|_{\varphi_2} \leq 2 \sqrt{2} \nuz \left\|\langle W_i, u \rangle_{\Hz}\right\|_{L_2(\Prob)}\,.
\end{equation}
We conclude that the $W_i$s are sub-Gaussian  with $B = 2\sqrt{2} \nuz$.
\paragraph{The $W_i$s are pre-gaussian.} We define $Z = \sqrt{\frac{\mz}{n}} \sum_{j=1}^n G_j \empcovzt^{-1/2} \psiz(z_j)$, with $G_j \overset{\text{i.i.d.}}{\sim} \bmN\left(0, 1/\mz\right)$. $Z$ is a Gaussian random variable that admits the same covariance operator as the $W_i$s. So, the $W_i$ are pre-Gaussian.

\paragraph{Applying concentration bound. } Because the $W_i$s are i.i.d. weakly square integrable centered random variables, we can apply \cref{th:kolt_subG}, and by using also \cref{lem:eci_wi}, and the condition $\mz \geq \max\left(6 \bmN_{\bmZ}^{\infty}(t), \log\left(1/\delta\right)\right)$, we obtain
\begin{equation}
        \left\|\empcovzt^{-1/2}\left(\empcovz-\sketchempcovz\right)\empcovzt^{-1/2}\right\|_{\op} \leq \mathfrak{C} \frac{2 \sqrt{2} \nuz \sqrt{6 \bmN_{\bmZ}^{\infty}(t)} + 8 \nuz^2 \sqrt{\log\left(1/\delta\right)}}{\sqrt{\mz}}\,.
\end{equation}

\end{proof}
\section{PROBABILISTIC BOUNDS}\label{apx:prob_bound}

In this section, we provide all the probabilistic bounds used in our proofs. In particular, we restate bounds from other works for the sake of providing a self-contained work.
We order them in the same in order of appearance in our proofs.



\begin{lemma}[Bound $(A.1) = \left\|\left(\SY^\# \SX - H_\lambda \empcovx\right) \covxlambda^{-1/2}\right\|_{\HS}$ {\citep[Theorem B.10]{ciliberto2020general}}]\label{lem:proba_bound_1} Let $\delta \in [0,1]$, $n \in \mathbb{N}$ sufficiently large such that $\lambda = n^{-1/(1+\gammax)} \geq \frac{9\kappax^2}{n} \log(\frac{n}{x})$ Under our set of assumptions, the following holds with probability at least $1-\delta$
\begin{align}
    (A.1) \leq M \log(4/\delta) n^{-\frac{1}{2(1 + \gammax)}}
\end{align}
where the constant $M$ depends on $\kappay, \|H\|_{\HS}, \delta$.
\end{lemma}
\begin{proof} This lemma can be obtained from {\citet[Theorem B.10]{ciliberto2020general}}, by noticing that the bound of Theorem B.10 is obtained by upper bounding the sum of $(A.1)$ and a positive term, such that the bound of {\citet[Theorem B.10]{ciliberto2020general}} is an upper bound of $(A.1)$.

\begin{lemma}[Bound $ \big\| \empcovzlambda^{-1/2}  \covzlambda^{1/2}\big\|_{\op}$ {\citep[Lemma 3.6]{rudi2013sample}}] \label{lem:proba_bound_2} If $\frac{9}{n} \log \frac{n}{\delta} \leq \lambda \leq \|\covz\|_{\op}$, then we have with probability $1- \delta$
\begin{align}
    \| \empcovzlambda^{-1/2}  \covzlambda^{1/2}\|_{\op} \leq  \sqrt{2}.
\end{align}

\end{lemma}

\end{proof}

\begin{lemma}[Bound $\big\|\covzlambda^{-1/2} \empcovzlambda^{1/2}\big\|_{\op}$]\label{lem:proba_bound_3} If $\frac{9}{n} \log \frac{n}{\delta} \leq \lambda \leq \|\covz\|_{\op}$, then with probability $1- \delta$
\begin{align}
    \big\|\covzlambda^{-1/2} \empcovzlambda^{1/2}\big\|_{\op} \leq  \sqrt{\frac{3}{2}}.
\end{align}

\end{lemma}
\begin{proof} We have
\begin{align}
     \big\|\covzlambda^{-1/2} \empcovzlambda^{1/2}\big\|_{\op} &= \big\|\covzlambda^{-1/2} \empcovzlambda \covzlambda^{-1/2}\big\|_{\op}^{1/2}\\
     &= \big\|I + \covzlambda^{-1/2} (\empcovz - \covz) \covzlambda^{-1/2}\big\|_{\op}^{1/2}\\
     &\leq \left(1 + \big\|\covzlambda^{-1/2} (\empcovz - \covz) \covzlambda^{-1/2}\big\|_{\op}\right)^{1/2}\\
     &\leq \sqrt{\frac{3}{2}}
\end{align}
with probability at least $1-\delta$, where the last inequality is from \citet[Lemma 3.6]{rudi2013sample}.

\begin{theorem}[sub-Gaussian concentration bound {\citep[Theorem 9]{Koltchinskii_2017}}]\label{th:kolt_subG}
    Let $W, W_1, \ldots, W_m$ be i.i.d. weakly square integrable centered random vectors in a separable Hilbert space $\Hz$ with covariance operator $\Sigma$. If $W$ is sub-Gaussian and pre-Gaussian, then there exists a constant $\mathfrak{C}>0$ such that, for all $\tau \geq 1$, with probability at least $1-e^{-\tau}$,
    \begin{equation}
        \|\hat{\Sigma}-\Sigma\| \leq \mathfrak{C} \|\Sigma\|\left(B \sqrt{\frac{\mathbf{r}(\Sigma)}{m}} \vee \frac{\mathbf{r}(\Sigma)}{m} \vee B^2 \sqrt{\frac{\tau}{m}} \vee B^2 \frac{\tau}{m}\right)\,,
    \end{equation}
    where $B > 0$ is the constant such that $\left\|\langle W, u \rangle_{\Hy}\right\|_{\varphi_2} \leq B \left\|\langle W, u \rangle_{\Hy}\right\|_{L_2(\Prob)}$ for all $u \in \Hz$.
\end{theorem}

\end{proof}
\section{AUXILIARY RESULTS AND DEFINITIONS}\label{apx:aux_res}

\begin{definition}\label{def:dim_eff} For $t>0$, we define the random variable 
\begin{align}
    \bmN(z, t) = \langle \psiz(z), \covzt^{-1} \psiz(z) \rangle_{\Hz}
\end{align}
with $z \in \bmZ$ distributed according to $\rhoz$ and let
\begin{equation}
    \deffz(t) = \E_z\left[\bmN(z, t)\right] = \Tr\left(\covz \covzt^{-1}\right)\,, \quad \bmN_{\bmZ}^{\infty}(t) = \sup_{z \in \bmZ} ~ \bmN(z, t)\,.
\end{equation}
We note $\bmN_{\bmX}^{\infty}, \deffx(t), \gammax, \Qy, \bmN_{\bmY}^{\infty}, \deffy(t), \gammay, \Qy$ for the input and output kernels $\kernelx, k _y$, respectively.
\end{definition}

\begin{lemma}\label{lem:int_dim} When Assumption \ref{asm:capacity} holds then we have
\begin{align}
    \deffz(t) \leq  \Qz t^{-\gammaz}.
\end{align}
    
\end{lemma}

\begin{proof} We have

\begin{align}
    \deffz(t) &= \Tr\left(\covz \covzt^{-1}\right)\\
    &\leq \Tr\left(\covz^{\gammaz}\right) \|\covz^{1-\gammaz}\covzt^{-1}\|_{\op}\\
    &\leq \Qz t^{-\gammaz}.
\end{align}

\end{proof}

\begin{lemma}\label{lem:emb} When Assumption \ref{asm:emb} holds then we have
\begin{align}
    \bmN_{\bmZ}^{\infty}(t) \leq  \bz \deffz(t) t^{-\muz}.
\end{align}
    
\end{lemma}

\begin{proof} We have

\begin{align}
    \bmN_{\bmZ}^{\infty}(t) &= \sup_{z \in \bmZ} ~ \langle \psiz(z), \covzt^{-1} \psiz(z) \rangle_{\Hz}\\
    &\leq \bz \Tr(\covzt^{-1} \covz^{1-\muz})\\
    &\leq \bz \Tr(\covzt^{-1}\covz) \|\covzt^{-\muz}\|_{\op}\\
    &\leq \bz \deffz(t) t^{-\muz}.
\end{align}
    
\end{proof}


We recall the following deterministic bound.

\begin{lemma}[Bound $\| \empcovxlambda^{1/2} \tilde{\eta}(\empcovx) \empcovxlambda^{1/2} \|_{\op}$ {\citep[Lemma 8]{rudi2015less}}]\label{lem:proba_bound_4} For any $\lambda >0$,
\begin{align}
    \| \empcovxlambda^{1/2} \tilde{\eta}(\empcovx) \empcovxlambda^{1/2} \|_{\op} \leq 1.
\end{align}
\end{lemma}

We introduce here some notations and definitions from \citet{Koltchinskii_2017}.
Let $W$ be a centered random variable in $\Hz$, $W$ is weakly square integrable iff $\left\|\langle W, u \rangle_{\Hz}\right\|_{L_2(\Prob)}^2 := \E\left[|\langle W, u \rangle_{\Hz}|^2\right] < + \infty$, for any $u \in \Hz$.
Moreover, we define the Orlicz norms. For a convex nondecreasing function $\varphi : \reals_{+} \rightarrow \reals_{+}$ with $\varphi(0)=0$ and a random variable $\eta$ on a probability space $\left(\Omega, \bmA, \Prob\right)$, the $\varphi$-norm of $\eta$ is defined as
\begin{equation}
    \left\|\eta\right\|_\varphi = \inf\left\{C > 0 : \E\left[\varphi\left(|\eta|/C\right)\right] \leq 1\right\}\,.
\end{equation}
The Orlicz $\varphi_1$- and $\varphi_2$-norms coincide to the functions $\varphi_1(u) = e^u - 1, u \geq 0$ and $\varphi_2(u) = e^{u^2} - 1, u \geq 0$.
Finally, \citet{Koltchinskii_2017} introduces the definitions of sub-Gaussian and pre-Gaussian random variables in a separable Banach space $E$. We focus on the case where $E = \Hz$.
\begin{definition}
    A centered random variable $X$ in $\Hz$ will be called sub-Gaussian iff, for all $u \in \Hz$, there exists $B > 0$ such that
    \begin{equation}
        \left\|\langle X, u \rangle_{\Hz}\right\|_{\psix_2} \leq B \left\|\langle X, u \rangle_{\Hz}\right\|_{L_2(\Prob)}\,.
    \end{equation}
\end{definition}
\begin{definition}
    A weakly square integrable centered random variable $X$ in $\Hz$ with covariance operator $\Sigma$ is called pre-Gaussian iff there exists a centered Gaussian random variable $Y$ in $\Hz$ with the same covariance operator $\Sigma$.
\end{definition}

\begin{lemma}[Expectancy, covariance, and intrinsic dimension of the $W_i$s]\label{lem:eci_wi} Defining  $W_i = \sqrt{\frac{\mz}{n}} \sum_{j=1}^n (\sketchz)_{i j} \empcovzt^{-1/2} \psiz(z_j) \in \Hz \quad \text{ for } i=1, \dots \mz$ where $\sketchz$ is a sub-Gaussian sketch, the following hold true
\begin{align}
    &\E_{\sketchz}\left[W_{i}\right] = 0\\
    &\Sigma = \E_{\sketchz}\left[W_i \otimes W_i\right] = \empcovzt^{-1/2} \empcovz \empcovzt^{-1/2}\\
    & \widehat \Sigma = \frac{1}{\mz} \sum_{i=1}^{\mz} \langle f, W_i \rangle_{\Hz} W_i = \empcovzt^{-1/2} \sketchempcovz \empcovzt^{-1/2}
\end{align}
and for $\delta \in \left(0, 1\right)$, if $\frac{9}{n}\log\left(\frac{n}{\delta}\right) \leq t \leq \left\|\covz\right\|_{\op} - \frac{9}{n}\log\left(\frac{n}{\delta}\right)$, then with probability $1-\delta$
\begin{align}
    r\left(\Sigma\right) = \frac{\E_{\sketchz}\left[\left\|X_i\right\|_{\Hz}\right]^2}{\left\|\Sigma\right\|_{\op}} \leq 6 \bmN_{\bmZ}^{\infty}(t)\,.
\end{align}
\end{lemma}

\begin{proof} First, it is straightforward to check that 
\begin{align}
    \frac{1}{\mz} \sum_{i=1}^{\mz} \langle f, W_i \rangle_{\Hz} W_i = \empcovzt^{-1/2} \sketchempcovz \empcovzt^{-1/2}.
\end{align}

Then, since $\E_{\sketchz}\left[(\sketchz)_{i:}\right]=0$,
\begin{equation}
    \E_{\sketchz}\left[W_{i}\right] = \sqrt{\frac{\mz}{n}} \empcovzt^{-1/2} \SZ^{\#} \E_{\sketchz}\left[(\sketchz)_{i:}\right] = 0\,.
\end{equation}

Then,
\begin{align}
    \left(W_i \otimes W_i\right)f &= \langle f, W_i \rangle_{\Hz} W_i \\
    &= \langle f, \sqrt{\mz} \empcovzt^{-1/2} \SZ^{\#} (\sketchz)_{i:} \rangle_{\Hz} \sqrt{\mz} \empcovzt^{-1/2} \SZ^{\#} (\sketchz)_{i:} \\
    &= \mz \left((\sketchz)_{i:}^\top \SZ \empcovzt^{-1/2} f\right) \empcovzt^{-1/2} \SZ^{\#} (\sketchz)_{i:} \\
    &= \empcovzt^{-1/2} \SZ^{\#} \left(\mz (\sketchz)_{i:} (\sketchz)_{i:}^\top\right) \SZ \empcovzt^{-1/2} f\,,
\end{align}
and since $\E_{\sketchz}\left[\mz (\sketchz)_{i:} (\sketchz)_{i:}^\top\right] = I_n$,
\begin{align}
    \Sigma &= \E_{\sketchz}\left[W_i \otimes W_i\right] \\&= \empcovzt^{-1/2} \SZ^{\#} \E_{\sketchz}\left[\mz (\sketchz)_{i:} (\sketchz)_{i:}^\top\right] \SZ \empcovzt^{-1/2} \\&= \empcovzt^{-1/2} \empcovz \empcovzt^{-1/2}\,.
\end{align}
Then,
\begin{align}
    \E_{\sketchz}\left[\left\|X_i\right\|_{\Hz}\right]^2 &\leq \E_{\sketchz}\left[\left\|X_i\right\|_{\Hz}^2\right] \quad \text{(by Jensen's inequality)} \\
    &= \mz \E_{\sketchz}\left[\langle \empcovzt^{-1/2} \SZ^{\#} (\sketchz)_{i:}, \empcovzt^{-1/2} \SZ^{\#} (\sketchz)_{i:}\rangle_{\Hz}\right] \\
    &= \frac{\mz}{n} \E_{\sketchz}\left[\langle \sum_{j=1}^n \sketchzij \psiz(z_j), \sum_{l=1}^n \sketchzil \empcovzt^{-1} \psiz(z_l) \rangle_{\Hz}\right] \\
    &= \frac{\mz}{n} \E_{\sketchz}\left[\sum_{j,l=1}^n \sketchzij \sketchzil \langle \psiz(z_j), \empcovzt^{-1} \psiz(z_l) \rangle_{\Hy}\right] \\
    &= \frac{\mz}{n} \sum_{j=1}^n \frac{1}{\mz} \langle \psiz(z_j), \empcovzt^{-1} \psiz(z_j) \rangle_{\Hz} \\
    &= \Tr\left(\empcovzt^{-1} \empcovz\right) \\
    &= \left\| \empcovzt^{-1/2} \empcovz^{1/2}\right\|_{\HS}^2 \\
    &\leq \left\| \empcovzt^{-1/2} \covzt^{1/2}\right\|_{\op}^2 \left\| \covzt^{-1/2} \empcovz^{1/2}\right\|_{\HS}^2\,.
\end{align}
But,
\begin{align}
    \left\| \covzt^{-1/2} \empcovz^{1/2}\right\|_{\HS}^2 &= \Tr\left(\covzt^{-1} \empcovz\right) \\
    &= \Tr\left(\covzt^{-1} \left(\frac{1}{n} \sum_{i=1}^n \psiz(z_i) \otimes \psiz(z_i)\right)\right) \\
    &= \frac{1}{n} \sum_{i=1}^n \Tr\left(\covzt^{-1} \left(\psiz(z_i) \otimes \psiz(z_i)\right)\right) \\
    &= \frac{1}{n} \sum_{i=1}^n \left\langle \psiz(z_i), \covzt^{-1} \psiz(z_i) \right\rangle_{\Hy} \\
    &= \frac{1}{n} \sum_{i=1}^n \bmN(z_i, t) \\
    &\leq \bmN_{\bmZ}^{\infty}(t)\,.
\end{align}
Then, from \cref{lem:proba_bound_2}, for $\delta \in \left(0, 1\right)$, and $\frac{9}{n}\log\left(\frac{n}{\delta}\right) \leq t \leq \left\|\covz\right\|_{\op}$, then with probability $1-\delta$,
\begin{align}
    \E_{\sketchz}\left[\left\|X_i\right\|_{\Hz}\right]^2 \leq 2 \bmN_{\bmZ}^{\infty}(t).
\end{align}
Then, $\left\|\Sigma\right\|_{\op} = \left\|\empcovzt^{-1/2} \empcovz^{1/2}\right\|_{\op}^2 \geq 1/3$ for $t \leq 2 \left\|\empcovz\right\|_{\op}$.

We conclude that
\begin{align}
    \frac{\E_{\sketchz}\left[\left\|W_i\right\|_{\Hz}\right]^2}{\left\|\Sigma\right\|_{\op}} \leq 6 \bmN_{\bmZ}^{\infty}(t).
\end{align}

Finally, in order to obtain a condition on $t$ that does not depend on empirical quantities, we use \cref{lem:proba_bound_2} which gives that, for any $\frac{9}{n}\log\left(\frac{n}{\delta}\right) \leq t^\prime \leq \left\|\covz\right\|_{\op}$, then $\covztprime \preceq 2 \empcovztprime$, which implies $2 \left\|\empcovz\right\|_{\op} \geq \left\|\covz\right\|_{\op} - t^\prime$. Now, taking $t^\prime = \frac{9}{n}\log\left(\frac{n}{\delta}\right)$, we obtain $\left\|\covz\right\|_{\op} - \frac{9}{n}\log\left(\frac{n}{\delta}\right) \leq 2 \left\|\empcovz\right\|_{\op}$.

\end{proof}

\section{CONTRIBUTIONS AND PREVIOUS WORKS}\label{subsec:related_work}

Excess-risk bounds for sketched kernel ridge regression have been provided in \citet{rudi2015less} in the case of Nyström subsampling, and scalar-valued ridge regression. Our proofs consist in similar derivations than in \citet{rudi2015less}. Nevertheless, we cannot apply directly their results in our setting. More precisely, we do the following additional derivations.
\begin{enumerate}
    \item Additional decompositions to deal with:
    \begin{enumerate}
        \item vector-valued regression instead of scalar-valued regression as in \citet{rudi2015less}
        \item input and output approximated feature maps
    \end{enumerate}
    \item Novel probabilistic bounds to deal with gaussian and sub-Gaussian sketching instead of Nyström sketching as in \citet{rudi2015less}.
\end{enumerate}


\section{ADDITIONAL EXPERIMENTS}
\label{apx:expes}

\subsection{Simulated Data Set for Least Squares Regression}
\label{apx:synthetic}

We report here some results about statistical performance on the synthetic data set described in \cref{sec:expes}.
First, we give an additional figure showing the MSE with respect to $\mx$ and $\my$ of the SISOKR model, see \cref{fig:mse_SISOKR}.
As reported in \cref{fig:mse_SIOKR_ISOKR}, SIOKR outperforms IOKR from $\mx = 100$, and ISOKR obtains very similar result to IOKR from $\my = 250$.

\begin{figure}[!ht]
\centering
%
\includegraphics[width=0.45\columnwidth]{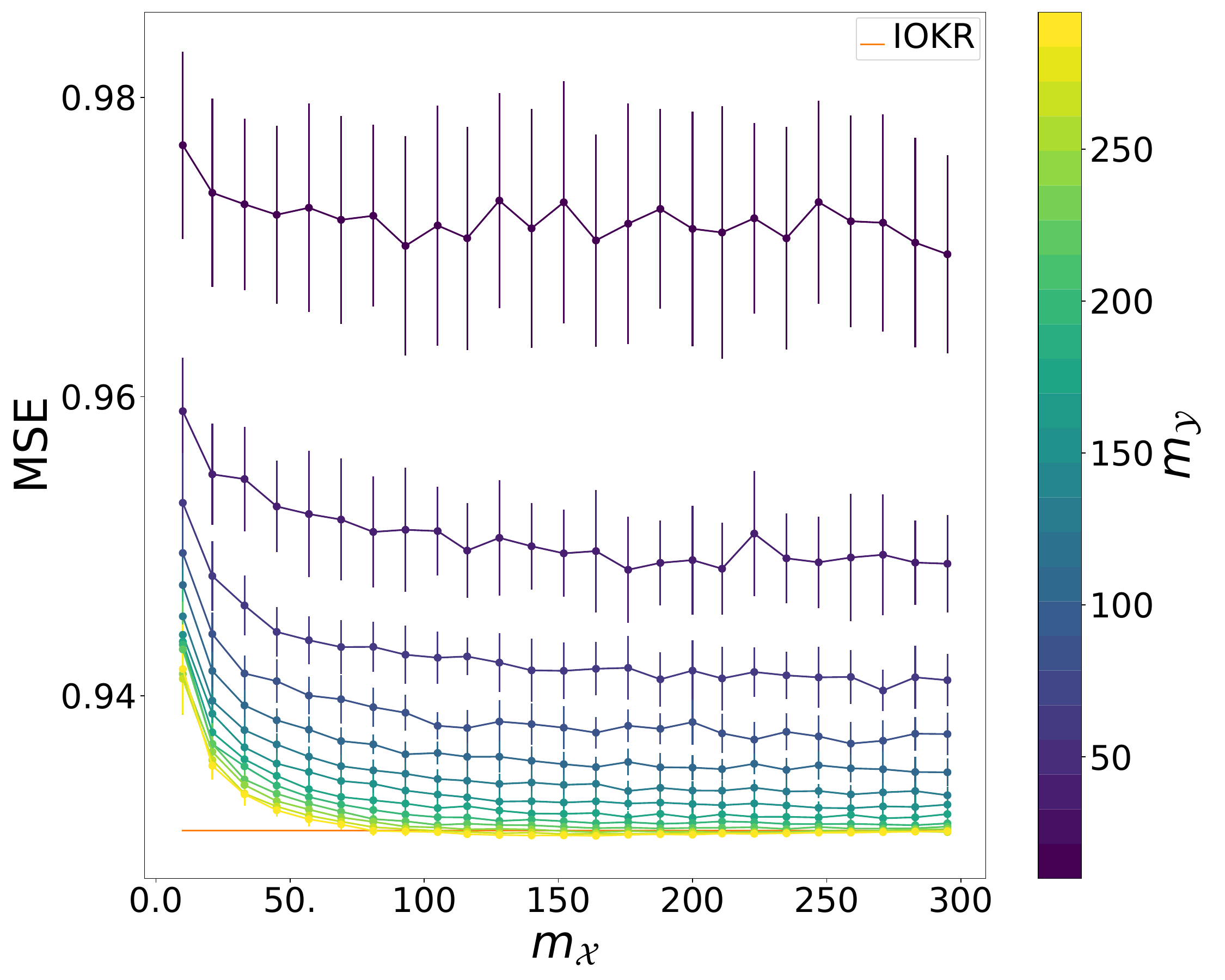}
\caption{Test MSE with respect to $\mx$ and $\my$ for the SISOKR model with $(2 \cdot 10^{-3})$-SR input and output sketches.}
\label{fig:mse_SISOKR}
\end{figure}

\begin{figure}[!ht]
\centering
\subfigure
{\label{fig:mse_SIOKR}\includegraphics[width=0.45\columnwidth]{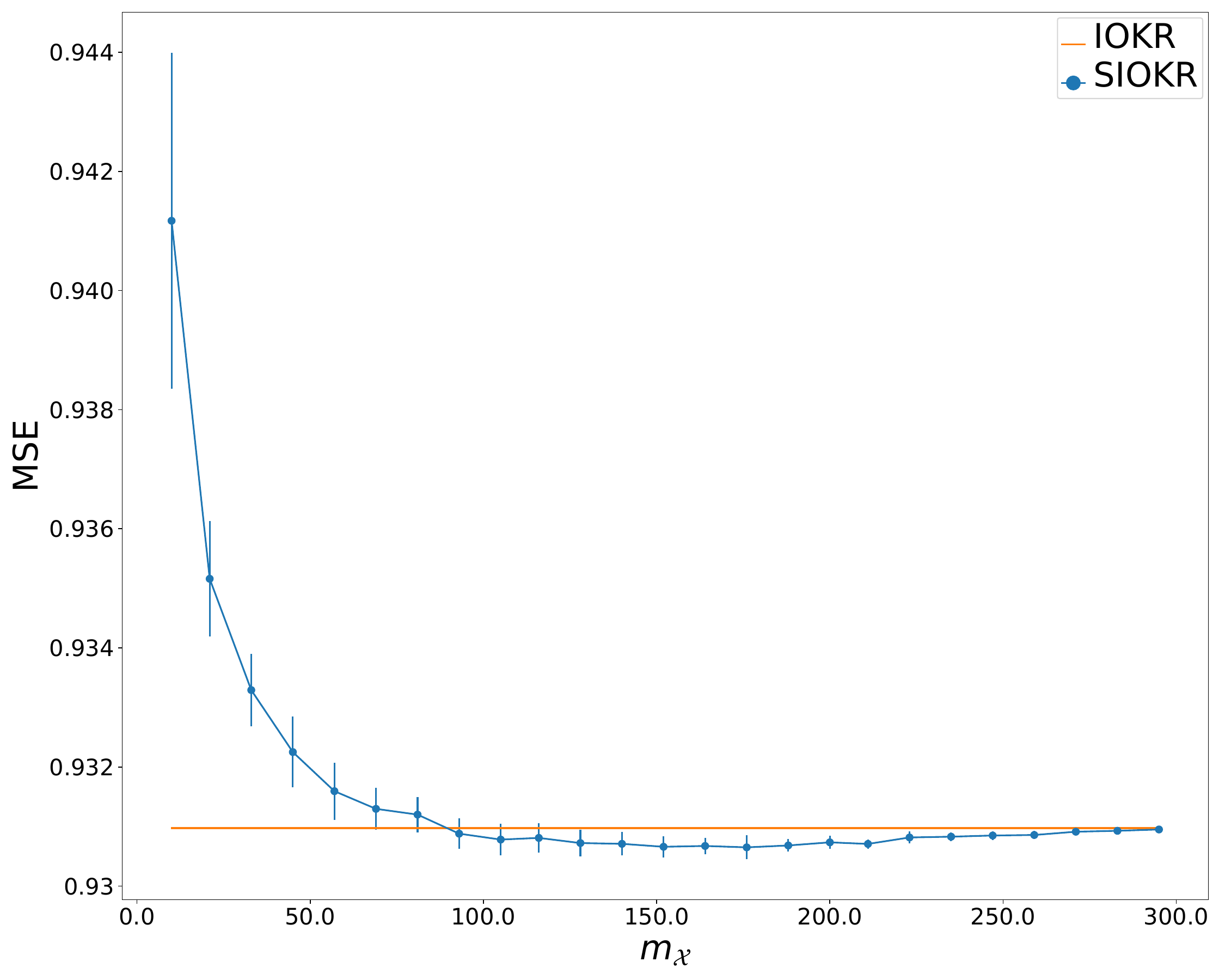}}
\qquad
\subfigure
{\label{fig:mse_ISOKR}\includegraphics[width=0.45\columnwidth]{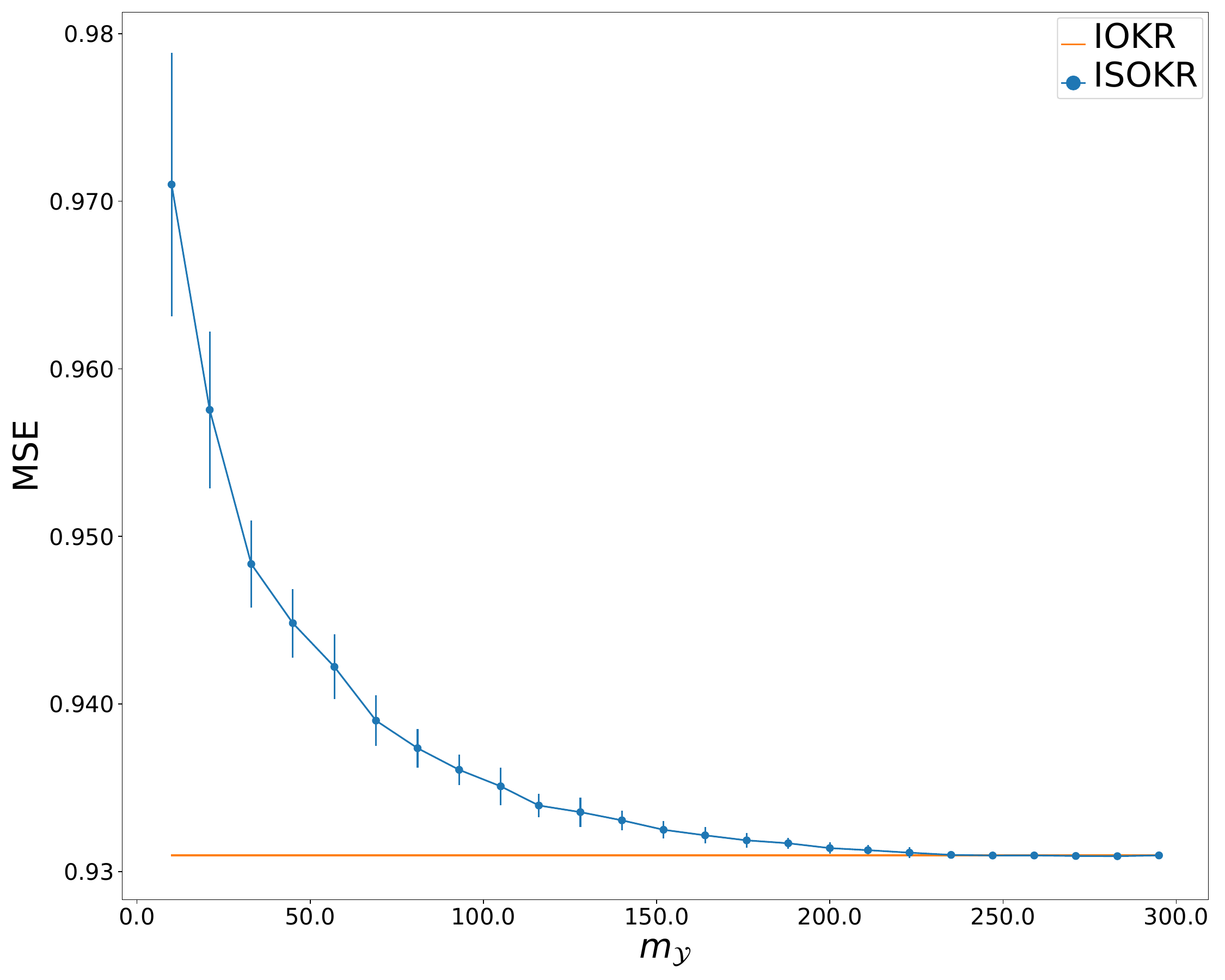}}
\caption{Test MSE with respect to $\mx$ and $\my$ for a SIOKR and ISOKR model respectively with $(2 \cdot 10^{-3})$-SR input and output sketches.}
\centering
\label{fig:mse_SIOKR_ISOKR}
\end{figure}

\subsection{More Details about Multi-Label Classification Data Set}
\label{apx:multi-label}

In this section, you can find more details about training and testing sizes, the number of features of the inputs, and the number of labels to predict of Bibtex, Bookmarks, and Mediamill data sets in \cref{table:multi-label-data-desc}.

\begin{table}[!ht]
\caption{Multi-label data sets description.}
\begin{adjustbox}{center}
\begin{small}
\begin{tabular}{ c c c c c c }
    \toprule
    Data set & $n$ & $n_{te}$ & $n_{features}$ & $n_{labels}$ \\ 
    \midrule
    Bibtex & $4880$ & $2515$ & $1836$ & $159$ \\
    Bookmarks & $60000$ & $27856$ & $2150$ & $298$ \\
    Mediamill & $30993$ & $12914$ & $120$ & $101$ \\
    \bottomrule
\end{tabular}
\end{small}
\end{adjustbox}
\label{table:multi-label-data-desc}
\end{table}

\end{document}